\documentclass{article}

\usepackage[preprint]{neurips_2026}
\usepackage[utf8]{inputenc}
\usepackage[T1]{fontenc}
\usepackage{amsmath,amssymb,amsfonts,amsthm,bm}
\usepackage{booktabs,multirow,array,tabularx}
\usepackage{graphicx}
\usepackage{float}
\usepackage{caption}
\usepackage{subcaption}
\usepackage{microtype}
\usepackage{xcolor}
\usepackage{enumitem}
\usepackage{url}
\usepackage{hyperref}
\usepackage{tikz}
\usetikzlibrary{arrows.meta,positioning,fit,calc}
\hypersetup{hidelinks}

\newtheorem{theorem}{Theorem}

\newtheorem{lemma}{Lemma}
\newtheorem{corollary}{Corollary}

\newcommand{\E}{\mathbb{E}}
\newcommand{\KL}{\mathrm{KL}}
\newcommand{\Tc}{T_c}
\newcommand{\Sc}{S_c}
\newcommand{\Gtau}{G_{\tau}}
\newcommand{\Tzero}{T_0}

\newcommand{\tightparagraph}[1]{\par\noindent\textbf{#1}\quad}

\title{A Testable Certificate for Constant Collapse in Teacher-Guided VAEs}
\author{
  Zegu Zhang\textsuperscript{1} \qquad  Jianhua Peng\textsuperscript{2} \qquad Jian Zhang\textsuperscript{3}\\
  \texttt{\{zeguzhang@outlook.com, 2018100913@niit.edu.cn, tsegoochang2000@gmail.com\}}\\
  \textsuperscript{1}Independent Researcher \quad \textsuperscript{2}school of computing, Southeast University \quad \textsuperscript{3}Independent Researcher
}

\begin{document}
\maketitle

\begin{abstract}
Posterior collapse in variational autoencoders is often diagnosed by its symptoms: a small KL term, a strong decoder, or weak use of the latent code. These signals are useful, but they do not define a collapse boundary. We study a concrete failure mode, input-independent constant collapse, and show that this case admits an exact threshold. For any fixed nonconstant teacher distribution \(T(\cdot\mid x)\), the best constant student is the dataset-average teacher distribution, and its alignment cost is the teacher mutual information \(I_T(X;T)\). Therefore, if a strictly latent-only raw witness achieves alignment loss below this value, with a safety margin, the witness cannot be constant in the input.

This identity turns a qualitative failure mode into a measurable one. In CIFAR-100 experiments with per-seed teacher search, full training stays on the certified side of the boundary, removing alignment drives the raw witness into the constant-student regime, and restarting from a collapsed checkpoint with alignment enabled restores the certificate. Tiny-ImageNet-200 fixed-target runs show the same prevention--collapse--rescue pattern across three independently searched teachers. Standard VAE-style baselines, including methods that preserve reconstruction quality or post-hoc predictability, remain negative under the raw certificate. The guarantee is intentionally narrow: it certifies that the matched nonconstant teacher-relative variation passes through the latent pathway, rather than claiming that all forms of posterior collapse have been ruled out.
\end{abstract}

\section{Introduction}

Posterior collapse in variational autoencoders (VAEs) is usually recognized only after the fact: the decoder stops using the latent code, the KL term becomes small, and the learned representation carries little information about the input. These symptoms are useful, but they do not by themselves define a boundary. In particular, they do not tell us when a run has merely weakened its latent channel and when it has entered a specific collapsed regime.

We study one such regime: input-independent constant collapse. In this case, the latent representation is effectively constant as a function of the input. A simple consequence is that any predictor that only sees the latent code must also be constant. This gives a direct way to probe the latent path. We attach a raw latent-only assignment head and ask whether it can match a fixed teacher distribution \(T(\cdot\mid x)\) better than any constant student could.

This comparison has a closed-form answer. For a fixed, nonconstant teacher family, the best constant student is the dataset-average teacher distribution. Its expected alignment cost is exactly the teacher mutual information \(I_T(X;T)\). Thus \(I_T\) is not a tuning heuristic or an empirical reference curve; it is the loss achieved by the strongest input-independent student. If the raw latent-only witness obtains alignment loss below this value, and in practice below \(I_T-\tau\), then the witness cannot be constant in the input. Since the witness receives no information except through \(z_\phi(x)\), the nonconstant variation it matches must pass through the latent representation.

The role of the teacher is deliberately limited. We use it as a fixed distributional reference, for example from a pretrained, self-supervised, or preselected clustering model. It is not an additional label source for the final task. More importantly, the certificate is never computed from a teacher-aware head. It uses only
\[
S^{raw}_\theta(\cdot\mid x)=\mathrm{softmax}(g_\theta(z_\phi(x))).
\]
This distinction is essential: a head with direct access to teacher-side information could remain input-dependent even when the latent code had collapsed, and would therefore say little about the latent pathway itself.

The resulting margin provides a practical diagnostic. A run whose raw witness stays on the certified side of the threshold has avoided this collapse mode. A no-alignment run, by contrast, can be tested for whether it falls back into the constant-student region. The same quantity also gives an operational target for recovery: starting from a collapsed checkpoint, reintroducing alignment should move the raw witness back across the boundary. Our experiments are therefore organized around a prevention--collapse--rescue pattern, rather than around endpoint reconstruction metrics alone.

\tightparagraph{Contributions.}
Our main contributions are as follows:
\begin{itemize}[leftmargin=*,noitemsep]
    \item \textbf{An exact constant-student threshold.}
    We show that, for a fixed teacher \(T(\cdot\mid x)\), the optimal input-independent student has alignment cost \(I_T(X;T)\). This gives a sharp boundary for input-independent constant collapse.
    \item \textbf{A raw latent-witness margin.}
    We turn this identity into a computable practical margin \(G_\tau\), evaluated only through a \(z\)-only raw head, for monitoring whether the latent pathway carries nonconstant teacher-relative variation.
    \item \textbf{A prevention--collapse--rescue protocol.}
    We use the margin not only as a diagnostic but also as an intervention target: no-alignment runs are driven into the collapsed region, and alignment is then reintroduced from the same checkpoint to test whether the raw witness can cross back.
    \item \textbf{Empirical validation across teacher protocols.}
    We evaluate the boundary on CIFAR-100 teacher-search runs and Tiny-ImageNet-200 fixed-target runs, and compare against standard VAE-style baselines that do not satisfy the raw certificate.
\end{itemize}

\section{Theoretical Framework: Certifying Constant Collapse}
\subsection{Teacher-guided latent VAE with a raw z-only witness}
Write \(T_x=T(\cdot\mid x)\in\Delta_K\) for the teacher assignment vector, where
\[
\Delta_K:=\{p\in[0,1]^K:\sum_{c=1}^K p_c=1\}.
\]
The model has two distinct pathways:
\[
X\to Z,\qquad X\to T(x),\qquad Z\to S^{raw},\qquad (Z,T(x))\to \hat X .
\]
It contains an encoder \(q_\phi(z\mid x)\), a raw latent-only student head \(g_\theta\), and a decoder \(p_\psi(x\mid z,T(x))\). The decoder may use both \(z\) and the teacher vector. The certificate, however, is restricted to
\[
S^{raw}_\theta(\cdot\mid x)=\mathrm{softmax}(g_\theta(z_\phi(x))).
\]
This restriction is the main design choice. If every input maps to the same latent code, a raw \(z\)-only head must output the same assignment vector for every input. A head that also receives \(x\) or \(T(x)\) would not have this property, and could therefore create apparent input dependence after the latent pathway had already collapsed. We use \emph{witness} for the raw \(z\)-only head and \emph{certificate} for the positive margin obtained from that head.

The training objective has the form
\[
\mathcal L=\mathcal L_{\mathrm{recon}}+\beta_z\mathcal L_{\mathrm{KL},z}+\lambda_{\mathrm{align}}L^{used}_{\mathrm{align}}+\lambda_{\mathrm{bal}}L_{\mathrm{bal}}.
\]
The balance term is a student-side regularizer. A typical choice is
\[
L_{\mathrm{bal}}=\KL\!\left(\bar S^{used}\,\|\,\mathrm{Unif}(K)\right),\qquad
\bar S^{used}=\E_x S^{used}(\cdot\mid x),
\]
which discourages the student from satisfying alignment by using only a small set of clusters. This term helps optimization, but it is not part of the constant-baseline theorem. The theorem quantity is always computed from the raw head:
\[
L^{raw}_{\mathrm{align}}=\E_x\KL\!\left(T(\cdot\mid x)\,\|\,S^{raw}_\theta(\cdot\mid x)\right).
\]
Some training variants also record a calibrated or protected ``used'' head. Those values describe the optimization objective. They are not used for certification.

\begin{figure}[H]
\centering
\resizebox{0.90\linewidth}{!}{%
\begin{tikzpicture}[>=Latex, font=\scriptsize]
\tikzstyle{box}=[draw, rounded corners, align=center, minimum height=0.72cm, inner sep=4pt]
\node[box, minimum width=1.20cm] (x) at (0,0) {Input\\$x$};
\node[box, minimum width=1.60cm] (enc) at (2,0) {Encoder\\$q_\phi(z\mid x)$};
\node[box, minimum width=1.15cm] (z) at (4,0) {Latent\\$z$};
\node[box, minimum width=2.10cm] (dec) at (6.3,0) {Decoder\\$p_\psi(x\mid z,T(x))$};
\node[box, minimum width=1.25cm] (rec) at (8.8,0) {Recon.\\$\hat x$};
\node[box, minimum width=2.45cm] (teach) at (6.3,1.65) {Teacher search/cache\\fixed $x\mapsto T(x)$};
\node[box, minimum width=2.25cm] (head) at (3.3,-1.55) {Raw z-only head\\$S^{raw}=\mathrm{softmax}(g_\theta(z))$};
\node[box, minimum width=2.45cm] (cert) at (7.2,-1.55) {Raw certificate\\$G_\tau=I(T)-L^{raw}_{align}-\tau$};
\draw[->] (x) -- (enc);
\draw[->] (enc) -- (z);
\draw[->] (z) -- (dec);
\draw[->] (dec) -- (rec);
\draw[->] (x.north) |- (teach.west);
\draw[->] (teach.south) -- (dec.north);
\draw[->] (z.south) -- (head.north);
\draw[->] (head.east) -- (cert.west);
\coordinate (teachdashstart) at (teach.east);
\coordinate (teachdashcornerA) at ($(teachdashstart)+(2.10,0)$);
\coordinate (teachdashcornerB) at ($(teachdashcornerA|-cert.east)$);
\draw[dashed] (teachdashstart) -- (teachdashcornerA);
\draw[dashed] (teachdashcornerA) -- (teachdashcornerB);
\draw[->, dashed] (teachdashcornerB) -- (cert.east);
\node[align=center] at (6.3,2.33) {teacher target fixed after search};
\node[align=center] at (5.25,-2.25) {certificate path is strictly z-only};
\end{tikzpicture}}
\caption{Teacher-guided latent VAE with a raw \(z\)-only certificate. Generation may use \((z,T(x))\), but certification uses only \(S^{raw}_\theta(\cdot\mid x)\). In fixed-target runs, \(T(x)\) is the cached target \(T_0(x)=\mathrm{GMM}(\mu_0(x))\), not an anchor on the current encoder.}
\label{fig:arch}
\end{figure}
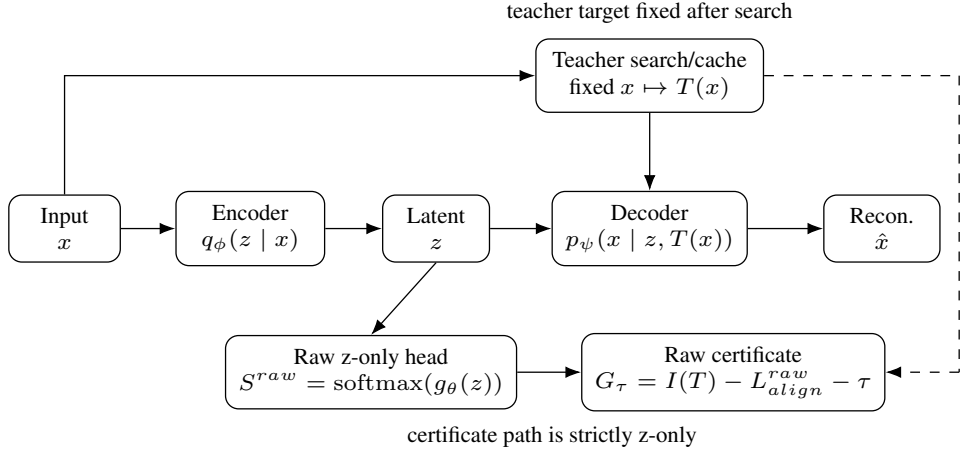

\subsection{Fixed informative targets and safety margins}
After warm-up, we collect encoder features and fit candidate GMM teachers. Each candidate gives soft assignments \(T_x\in\Delta_K\) and a teacher mutual information
\[
I_T(X;T)=\E_x\KL(T_x\|\bar T),\qquad \bar T=\E_x[T_x].
\]
Teacher selection uses finite-sample diagnostics: teacher MI, mean top-1 margin, high-margin fraction, hard and soft usage balance, and minimum component mass. These checks help avoid uninformative targets. They are not additional theorem assumptions; the theorem only needs the teacher family to be fixed and nonconstant. The selected teacher is constructed after a standard warm-up stage, so the full, no-alignment, and rescue variants share the same teacher artifact rather than fitting different references after each endpoint.

For larger experiments, recomputing \(T_\theta(x)=\mathrm{GMM}(\mu_\theta(x))\) throughout training can move the diagnostic together with the encoder. We therefore also use a fixed-target protocol. After teacher search, we cache
\[
T_0(x_i)=\mathrm{GMM}(\mu_0(x_i))
\]
and hold these targets fixed. The encoder remains trainable; we do not freeze it or add an anchoring penalty such as \(\|\mu_\theta(x)-\mu_0(x)\|^2\). The fixed-target margin is
\[
G^{T_0}_\tau=I(T_0)-\mathcal L_{\mathrm{align}}(T_0,S_\theta)-\tau.
\]
A positive \(G^{T_0}_\tau\) has the same teacher-relative interpretation: the raw student cannot be input-independent constant with respect to the fixed informative target.

\subsection{Exact constant-baseline identity}
Let \(T_x=T(\cdot\mid x)\) and \(\bar T=\E_x[T_x]\). The constant baseline is exact.

\begin{lemma}[Constant-baseline decomposition]
For any constant student assignment \(\alpha\in\Delta_K\),
\[
\E_x\KL(T_x\|\alpha)=I_T(X;T)+\KL(\bar T\|\alpha).
\]
Consequently, \(\inf_{\alpha\in\Delta_K}\E_x\KL(T_x\|\alpha)=I_T(X;T)\), with minimizer \(\alpha=\bar T\) whenever \(\bar T\) has full support.
\end{lemma}
\begin{proof}
Expand the KL term and add--subtract \(\log\bar T(c)\):
\[
\E_x\sum_cT_x(c)\log\frac{T_x(c)}{\alpha(c)}=
\E_x\sum_cT_x(c)\log\frac{T_x(c)}{\bar T(c)}+\sum_c\bar T(c)\log\frac{\bar T(c)}{\alpha(c)}.
\]
The two terms are \(I_T(X;T)\) and \(\KL(\bar T\|\alpha)\).
\end{proof}

\tightparagraph{Intuition.}
A constant student can only predict the dataset-average teacher assignment. Any raw \(z\)-only student that beats this average-teacher baseline must vary across inputs. Since the witness is strictly \(z\)-only, that variation cannot be supplied by a teacher-aware side channel.

\begin{lemma}[Constant latent codes force a constant z-only head]
If \(S^{raw}_\theta(\cdot\mid x)=h_\theta(z_\phi(x))\) and \(z_\phi(x)\equiv z_0\), then \(S^{raw}_\theta(\cdot\mid x)\equiv h_\theta(z_0)\).
\end{lemma}

\begin{theorem}[Raw alignment below \(I_T\) excludes a constant student]
Assume \(S^{raw}_\theta(\cdot\mid x)=h_\theta(z_\phi(x))\). If
\[
L^{raw}_{\mathrm{align}}:=\E_x\KL(T(\cdot\mid x)\|S^{raw}_\theta(\cdot\mid x))<I_T(X;T),
\]
then \(S^{raw}_\theta(\cdot\mid x)\) cannot be constant in \(x\).
\end{theorem}
\begin{proof}
If \(S^{raw}_\theta(\cdot\mid x)\equiv\alpha\), Lemma 1 gives \(L^{raw}_{\mathrm{align}}=\E_x\KL(T_x\|\alpha)\ge I_T(X;T)\), a contradiction.
\end{proof}

\begin{corollary}[Exclusion of input-independent constant latent collapse]
Under the theorem assumptions, \(L^{raw}_{\mathrm{align}}<I_T(X;T)\) implies that \(z_\phi(x)\) is not constant in \(x\).
\end{corollary}

The logical chain is
\[
z_\phi(x)\equiv z_0\quad\Rightarrow\quad S^{raw}_\theta(\cdot\mid x)\equiv \alpha\quad\Rightarrow\quad L^{raw}_{\mathrm{align}}\ge I_T(X;T).
\]
The certificate is the contrapositive. It is intentionally narrower than a lower bound on \(I(X;Z)\): it does not certify semantic quality, linear separability, or downstream optimality. It certifies only that this constant-code failure mode has been ruled out. In finite-data experiments we use
\[
G_\tau=I_T-L^{raw}_{\mathrm{align}}-\tau,
\]
with \(\tau=0.1\) by default. Thus a certificate-positive endpoint lies below the exact constant baseline by a fixed safety margin. Appendix results include near-boundary cases where the bare margin is positive but the practical margin is not, which is why we report \(G_\tau\) rather than only \(I_T-L^{raw}_{\mathrm{align}}\).

\begin{corollary}[Teacher-relative property transfer]
Let \(a(T_x)\) be a nonconstant statistic of the fixed teacher distribution, such as a cluster preference, a coarse partition, or a self-supervised attribute encoded in \(T(\cdot\mid x)\). If the raw latent-only witness matches the teacher below the constant-student threshold, \(L^{raw}_{\mathrm{align}} < I_T(X;T)\), then the matched nonconstant teacher-relative variation cannot be produced by an input-side teacher channel. It must be mediated by the latent path \(x\mapsto z_\phi(x)\mapsto S^{raw}(z_\phi(x))\). With the practical margin, the same statement holds under the stronger condition \(G_\tau>0\).
\end{corollary}

\begin{proof}[Proof sketch]
If the latent path were input-independent constant, Lemma~2 would force the raw \(z\)-only witness to be constant. Lemma~1 would then give \(L^{raw}_{\mathrm{align}}\ge I_T(X;T)\), contradicting \(L^{raw}_{\mathrm{align}}<I_T(X;T)\). Therefore any teacher-relative variation matched below the constant baseline must be available through the latent pathway.
\end{proof}

\tightparagraph{Interpretation.}
The statement differs from a post-hoc probe. A probe trained after representation learning may find information somewhere in the representation, whereas the certificate tests the training-time raw latent path itself. The statement is also weaker than claiming transfer of all human-interpretable semantics. The transferred object is exactly the nonconstant structure present in the chosen teacher distribution and actually matched by the raw witness. The teacher is therefore more than an auxiliary loss: it defines the measurable property whose nonconstant component must pass through the latent pathway whenever the raw certificate is positive.

\section{Method}

The certificate is only meaningful under a controlled evaluation protocol: the teacher must be fixed, and the reported quantity must be computed from a head that sees only the latent code. This section describes the corresponding training procedure. The goal is not to introduce an additional model class, but to isolate when the latent pathway can and cannot support a nonconstant raw witness. Figure~\ref{fig:arch} gives the main architectural view, and Appendix Table~\ref{tab:repro_summary} records the full run-level implementation details.

\tightparagraph{Warm-up and teacher construction.}
We begin with a short VAE warm-up, before using any fixed teacher certificate. The warm-up provides the feature representation on which candidate GMM teachers are fitted. We then score the candidates, choose one informative teacher distribution \(T(\cdot\mid x)\), and keep it fixed for the downstream comparison.

This distinction matters for how the margin is interpreted. In the teacher-search setting, the selected teacher and its empirical threshold \(\widehat I_T\) are defined per run. In the fixed-\(T_0\) setting, used for the Tiny-ImageNet-200 experiments, teacher search is performed independently for each seed and the resulting targets \(T_0(x)\) are cached before training the downstream variants. The encoder remains trainable; only the teacher targets are frozen.

\tightparagraph{Training modes.}
We compare four modes. \emph{Full} training uses reconstruction, latent KL, alignment, and balance terms. \emph{No-alignment} removes the explicit alignment term and serves as the induced-collapse control. \emph{Rescue} starts from the endpoint of the corresponding no-alignment run and resumes training with alignment enabled. \emph{Fixed-\(T_0\)} training uses cached targets \(T_0(x)\) for the full, no-alignment, and rescue variants.

The rescue setting is intentionally strict. It does not restart from a fresh initialization, but from a checkpoint that has already entered the collapsed region according to the raw witness. A successful rescue therefore shows that reintroducing alignment can move the existing model back across the boundary, rather than merely that a new run can avoid collapse.

\tightparagraph{Controlled comparisons.}
Across these modes, the encoder, decoder, latent dimensionality, and raw \(z\)-only assignment head are kept fixed. The variants differ only in whether the alignment term is active, whether training starts from a fresh initialization or from a completed no-alignment checkpoint, and whether the target is the selected teacher \(T(x)\) or a cached fixed target \(T_0(x)\). This keeps the prevention--collapse--rescue comparison tied to the training intervention rather than to changes in architecture or model capacity.

\tightparagraph{Optimization protocol.}
Each run follows the same high-level sequence: warm up the VAE, fit and score teacher candidates on warm-up features, select a fixed target, and then train one of the downstream variants. For rescue, the initial checkpoint is the endpoint of the paired no-alignment run.

At each reporting point, we compute the certificate quantities from the raw head:
\[
S^{raw}_\theta(\cdot\mid x)=\mathrm{softmax}(g_\theta(z_\phi(x))).
\]
We report the raw alignment loss \(L^{raw}_{\mathrm{align}}\), the empirical teacher threshold \(\widehat I_T\), and the practical margin \(G_\tau\). For fixed-target experiments we additionally report \(G^{T_0}_\tau\). Appendix~\ref{app:experimental_protocol} summarizes the protocols, saved artifacts, schedules, and reporting conventions.

\tightparagraph{Raw and used heads.}
All main claims are evaluated using the raw \(z\)-only head. The implementation also supports optional training-side protections, but these are optimization aids rather than part of the theorem or the main evaluation. We only allow protections that preserve constancy: if all raw rows are equal, then the protected rows must also remain equal. This prevents an auxiliary head from creating artificial input dependence after the latent path has collapsed. Appendix~\ref{app:detailed_system_overview} and Figure~\ref{fig:arch_appendix} describe this optional branch. The CIFAR-100 and Tiny-ImageNet-200 results reported in the main text use the raw-head/plain-objective path unless explicitly stated otherwise.

\begin{center}
\setlength{\fboxsep}{5.5pt}
\fbox{\begin{minipage}{\dimexpr\linewidth-2\fboxsep-2\fboxrule\relax}
\small
\textbf{Algorithm 1: Certificate-guided training and rescue.}

\vspace{0.15em}
\begin{enumerate}[leftmargin=1.35em,label=\arabic*.,topsep=2pt,itemsep=1.2pt,parsep=0pt,partopsep=0pt]
\item Warm up \(q_\phi\), \(p_\psi\), and the raw assignment head.
\item Fit candidate GMM teachers on warm-up features; choose \(T(x)\), or cache \(T_0(x)\) for fixed-target runs.
\item Train the full, no-alignment, rescue-from-no-alignment, or fixed-\(T_0\) variant.
\item At each report point, compute the raw certificate from \(S^{raw}\): \(L^{raw}_{\mathrm{align}}\), \(\widehat I_T\), \(G_\tau\), and, if applicable, \(G^{T_0}_\tau\).
\item Use optional protection variants only as training aids; the reported certificate remains raw-head based.
\end{enumerate}
\end{minipage}}
\end{center}
\vspace{0.2em}

\tightparagraph{Teacher-quality diagnostics.}
The constant-student identity only requires a fixed nonconstant teacher family. In finite-sample experiments, however, a poor teacher can make the diagnostic uninformative. We therefore score candidate teachers by teacher mutual information, mean top-1 margin, high-margin fraction, hard and soft usage balance, and minimum component mass. These filters remove nearly empty components and other degenerate targets. They are not assumptions of the theorem. A failed filter indicates a teacher-construction issue, not a failure of the constant-baseline identity; conversely, results obtained under an infeasible teacher are treated as stress cases rather than pooled into the strict-feasible mean.

\tightparagraph{Reporting convention.}
The certificate values in the paper are always computed from the raw \(z\)-only head. Protected heads, calibrated heads, linear probes, PSNR, active units, and post-hoc classifiers are reported only as diagnostics. They can help interpret a run, but they do not replace the raw teacher-relative margin.

\tightparagraph{Why rescue can work.}
Rescue is an optimization procedure rather than a theorem-level guarantee. At a constant student \(S_x\equiv\bar T\), KL alignment has logit descent direction \(T_x-\bar T\). Thus a nonconstant teacher supplies an input-dependent signal; if the encoder and raw head remain locally responsive to it, alignment can move the raw witness away from the constant basin. Appendix~\ref{app:guarded-dynamics} gives the corresponding continuous-time sufficient condition and failure mode.

\section{Experiments}
\label{sec:experiments}

We use the experiments to test whether the constant-student threshold actually organizes training behavior. The evaluation is built around three questions. First, does the full objective keep the raw latent witness on the certified side of the boundary? Second, does removing alignment drive the witness back toward the constant-student regime? Third, once this happens, can alignment recover the certificate from the same collapsed checkpoint?

\tightparagraph{Protocol and metrics.}
All certificate quantities are computed from the raw \(z\)-only head. We report the empirical teacher threshold \(\widehat I_T\), the raw alignment loss \(\widehat L^{raw}_{\mathrm{align}}\), and the practical margin
\[
G_\tau=\widehat I_T-\widehat L^{raw}_{\mathrm{align}}-0.1 .
\]
For fixed-target experiments, we instead report
\[
G^{T_0}_\tau=I(T_0)-L_{\mathrm{align}}(T_0,S^{raw})-0.1 .
\]
PSNR, active units, linear probes, and post-hoc heads are included only as diagnostics. They can reveal useful side information about a representation, but they do not replace the raw teacher-relative certificate.

\tightparagraph{What each experiment tests.}
CIFAR-100 is used as a teacher-search stress test. For each reported run, we select a teacher from warm-up features and evaluate the raw witness against the teacher used in that run. The resulting threshold \(\widehat I_T\) is therefore run-specific and is reported explicitly. Tiny-ImageNet-200 is used as a larger fixed-target test: for each seed, we search for a teacher, cache \(T_0(x)\), and then run the full, no-alignment, and rescue variants against that same cached target. Finally, the baseline comparison asks whether common VAE-style anti-collapse heuristics satisfy the same raw certificate without the proposed alignment mechanism.

\tightparagraph{Reconstruction is not enough.}
A recurring pattern in the results is that reconstruction quality and the raw certificate can disagree. No-alignment runs can obtain competitive, and sometimes better, PSNR while the raw margin is negative and the student MI is nearly zero. This is not a contradiction. A sufficiently expressive decoder can reconstruct through routes that do not require the raw latent assignment to carry teacher-relative input variation. The certificate is designed to test this narrower property directly.

\subsection{CIFAR-100 teacher-search stress test}

Table~\ref{tab:cifar100-search} summarizes the CIFAR-100 runs. Seeds 0 and 1 pass the strict teacher-feasibility filter. Seed 2 fails the minimum-mass criterion, so we treat it as a stress case rather than include it in a strict-feasible aggregate. The qualitative pattern is nevertheless the same in all three seeds: full training gives a positive raw margin, no-alignment drives the witness into a near-constant regime, and rescue restores a positive margin from the no-alignment checkpoint.

This experiment also separates teacher quality control from the certificate itself. The feasibility flag describes the teacher-construction filter; the certificate is then evaluated against the fixed teacher selected for that particular run. Since this is not the fixed-\(T_0\) protocol, \(\widehat I_T\) should not be read as a single shared threshold across variants. It is the empirical threshold for the teacher used in each run.

\begin{table}[t]
\centering
\small
\caption{CIFAR-100 teacher-search endpoints. The margin is computed as
\(G_\tau=\widehat I_T-\widehat L^{raw}_{\mathrm{align}}-0.1\).
Full and rescue runs have positive raw margins; no-alignment runs are negative. Seed~2 fails the strict feasibility filter because of the minimum-mass criterion, but the same prevention--collapse--rescue pattern remains visible.}
\label{tab:cifar100-search}
\setlength{\tabcolsep}{3.5pt}
\begin{tabular*}{\linewidth}{@{\extracolsep{\fill}}llccccc@{}}
\toprule
Seed & Teacher QC & Variant & \(\widehat I_T\) &
\(\widehat L^{raw}_{\mathrm{align}}\) & \(G_\tau\) & Student MI \\
\midrule
0 & feasible & Full    & 0.6924 & 0.0382 &  0.5542 & 0.6324 \\
0 & feasible & NoAlign & 0.6389 & 4.6051 & -4.0662 & \(6.9{\times}10^{-6}\) \\
0 & feasible & Rescue  & 0.6983 & 0.0595 &  0.5388 & 0.6607 \\
\midrule
1 & feasible & Full    & 0.7774 & 0.0662 &  0.6112 & 0.6976 \\
1 & feasible & NoAlign & 0.6024 & 4.6051 & -4.1027 & \(6.0{\times}10^{-6}\) \\
1 & feasible & Rescue  & 0.7476 & 0.0777 &  0.5699 & 0.6395 \\
\midrule
2 & stress case & Full    & 0.2693 & 0.0404 &  0.1289 & 0.2676 \\
2 & stress case & NoAlign & 0.1603 & 4.6052 & -4.5449 & \(1.1{\times}10^{-5}\) \\
2 & stress case & Rescue  & 0.4294 & 0.0485 &  0.2809 & 0.4209 \\
\bottomrule
\end{tabular*}
\end{table}

\begin{figure}[t]
\centering
\includegraphics[width=0.96\linewidth]{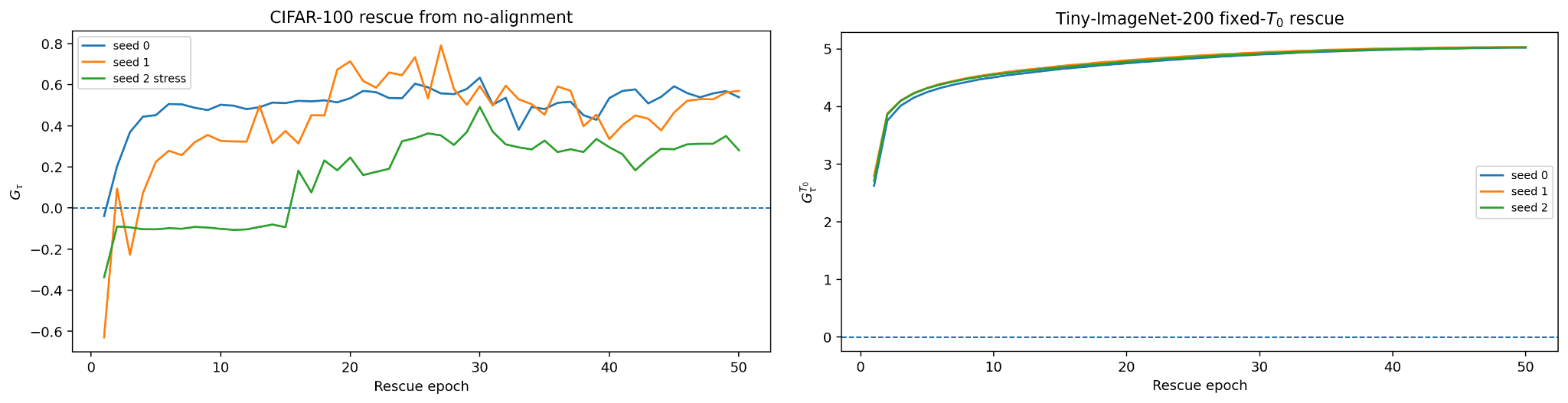}
\caption{Compact rescue trajectories. Rescue begins from collapsed checkpoints and restores positive raw margins on CIFAR-100 and Tiny-ImageNet-200.}
\label{fig:rescue-curves}
\vspace{-0.5em}
\end{figure}

\subsection{Tiny-ImageNet-200 fixed-target result}

Tiny-ImageNet-200 tests the same boundary at a larger scale under a fixed-target protocol. For each seed, we search for a teacher, cache \(T_0(x)\), and then train the full, no-alignment, and rescue variants using that cached target. The encoder remains trainable, and we do not use anchor regularization. Because the target is fixed, the relevant quantity is \(G^{T_0}_\tau\), not a current-teacher or dynamically reselected margin.

Table~\ref{tab:tiny-fixed} shows a clean separation. Full training is strongly positive across all three seeds. No-alignment stays near the fixed-target constant boundary and has essentially zero student MI. Rescue returns to the certified side of the boundary, with margins comparable to full training. The result mirrors CIFAR-100, but in a setting where the teacher target is cached before downstream training.

\begin{table}[t]
\centering
\small
\caption{Tiny-ImageNet-200 fixed-\(T_0\) endpoint summary over three seeds. All three teachers pass the feasibility filter. Full and rescue are strongly certificate-positive; no-alignment remains near the fixed-target constant boundary.}
\label{tab:tiny-fixed}
\setlength{\tabcolsep}{4pt}
\begin{tabular*}{\linewidth}{@{\extracolsep{\fill}}lccc@{}}
\toprule
Variant & \(G^{T_0}_\tau\) mean\(\pm\)std &
Student MI mean\(\pm\)std & Align mean\(\pm\)std \\
\midrule
Full    & \(5.0249\pm0.0083\)  & \(4.8687\pm0.0080\) & \(0.0725\pm0.0067\) \\
NoAlign & \(-0.0152\pm0.0029\) & \((6.3\pm2.1){\times}10^{-5}\) & \(5.1126\pm0.0068\) \\
Rescue  & \(5.0267\pm0.0063\)  & \(4.8425\pm0.0130\) & \(0.0707\pm0.0050\) \\
\bottomrule
\end{tabular*}
\end{table}

\subsection{Baseline comparison on CIFAR-100}

We next ask whether standard anti-collapse heuristics already satisfy the raw teacher-relative certificate. Table~\ref{tab:cifar100-baselines} suggests that they do not. Several baselines preserve active units or achieve reasonable reconstruction quality, and a post-hoc classifier can sometimes recover partial teacher predictability. However, every raw training-time margin remains negative.

The distinction between raw and post-hoc measurements is important. A representation may contain information that a separately trained classifier can extract, while the raw head used during training still fails to beat the constant-student threshold. Our certificate concerns the latter quantity: whether the training-time \(z\)-only witness carries enough teacher-relative input variation to cross the boundary.

\begin{table}[t]
\centering
\small
\caption{CIFAR-100 baseline diagnostics under the seed-0 teacher. All listed baselines, including VQ-VAE, have negative raw \(G_\tau\). Post-hoc predictability improves in some cases but does not replace the raw training-time certificate.}
\label{tab:cifar100-baselines}
\setlength{\tabcolsep}{4pt}
\begin{tabular*}{\linewidth}{@{\extracolsep{\fill}}lcccc@{}}
\toprule
Method & Raw \(G_\tau\) & Raw MI & Posthoc \(G_\tau\) & Posthoc MI \\
\midrule
Vanilla VAE & -4.0956 & 0.1176 & -1.9233 & 0.2389 \\
\(\beta\)-VAE (\(10^{-5}\)) & -4.1813 & 0.1487 & -0.7503 & 0.2746 \\
\(\beta\)-VAE (\(10^{-3}\)) & -4.0059 & 0.0235 & -1.7958 & 0.0777 \\
KL annealing & -3.4838 & 0.0235 & -1.3313 & 0.0953 \\
Free bits & -4.1183 & 0.1167 & -1.9925 & 0.2314 \\
Minimal GMVAE & -7.3780 & 1.9902 & -1.4372 & 0.5273 \\
VQ-VAE & -4.2482 & 0.0066 & -1.6583 & 0.2001 \\
\bottomrule
\end{tabular*}
\end{table}

Taken together, the experiments support the intended interpretation of the margin. CIFAR-100 shows prevention, induced collapse, and rescue under teacher search. Tiny-ImageNet-200 shows the same pattern under cached fixed targets. Standard VAE-style baselines remain raw-certificate negative under the same teacher-relative test. Full curves, CIFAR-10 sanity checks, the fixed-teacher CIFAR-100 control, and teacher diagnostics are provided in the appendix.

\section{Related Work}
Posterior collapse in VAEs is well documented, especially when expressive decoders learn to model the data while ignoring the latent variable \citep{kingma2014,bowman2016,chen2017,he2019lagging}. Several approaches modify the ELBO, tune information pressure, or encourage latent usage \citep{alemi2018,zhao2019}. Other work changes the latent structure itself, for example through mixture or discrete representations \citep{dilokthanakul2016,oord2017}. These methods are training interventions or architectural alternatives. Our contribution is different: we isolate one collapse mode, derive its exact constant-student boundary under a fixed teacher, and use that boundary as a certificate and rescue target.

Teacher guidance and distillation typically use a teacher to improve a student's predictive behavior \citep{hinton2015}. Here the teacher has a narrower role. It is a fixed self-supervised, pretrained, or preselected distributional reference used to make a collapse boundary measurable. The raw latent-only student is the witness; teacher-aware paths are excluded from certification. This separation is what preserves the implication from constant latent codes to constant raw assignments.

The certificate also differs from post-hoc probing. A probe trained after representation learning can show that some information is recoverable from a representation, but it does not show that the training-time latent pathway avoided the constant-student region. For this reason, our baseline table reports both raw and post-hoc quantities. Post-hoc predictability may be nonzero even when the raw certificate is negative.

\section{Limitations and Discussion}
The certificate targets a specific failure mode: input-independent constant collapse of the raw latent witness. It does not certify semantic optimality, label alignment, downstream state-of-the-art performance, or freedom from every form of posterior collapse. This narrower scope is what makes the result exact. A fixed teacher gives a stable distributional reference, and a strictly latent-only witness makes the constant-collapse boundary computable.

The approach also depends on teacher quality. The theorem requires only a fixed nonconstant teacher, but the empirical value of the certificate improves when the teacher is informative, balanced, and stable. We therefore report teacher feasibility as quality control rather than as a theorem assumption. The CIFAR-100 seed~2 stress case is kept separate for this reason: it is useful evidence about training dynamics under a weaker teacher, but it should not be treated as part of the strict-feasible mean.

Teacher and student balance do not change the certificate itself, but they do affect how informative it is in practice. The constant-student identity only assumes that the teacher is fixed and nonconstant. If the teacher is highly imbalanced, however, the test becomes less discriminative: a constant student can already explain much of the teacher mass by predicting the dominant components.

The same issue can occur on the student side. A positive raw margin rules out a constant witness, but it does not imply that every teacher mode is represented. If the raw student uses only a small subset of components, the latent path may still exhibit teacher-relative partial collapse. We therefore report teacher feasibility and usage-balance diagnostics separately from the theorem-level certificate. When the teacher is informative and the raw student is both aligned and non-degenerate in usage, the positive margin gives stronger evidence against this partial-collapse pattern, while remaining relative to the chosen teacher.

\tightparagraph{Toward quantitative collapse control.}
The teacher is not an additional label source; it is a device for making one collapse boundary visible. Once a boundary is measurable, it can be used by guarded schedules, rescue procedures, adaptive teacher refinement, or curriculum strategies that try to keep the raw latent pathway on the certified side. The present work establishes this control surface for constant collapse. Extending the idea to partial, semantic, or decoder-bypass collapse will require new witnesses and new baselines.

\section{Conclusion}
We introduced an exact certificate for one concrete form of VAE posterior collapse. For a fixed nonconstant teacher, the best constant student has alignment cost \(I_T(X;T)\). A raw latent-only witness whose alignment loss falls below this threshold, with margin \(G_\tau>0\), cannot be constant in the input. This turns input-independent constant collapse from a qualitative symptom into a measurable event.

Empirically, the threshold organizes the observed training dynamics. On CIFAR-100 and Tiny-ImageNet-200, full training stays on the certified side, no-alignment moves into the collapsed region, and rescue from a collapsed checkpoint restores the raw margin. Standard VAE-style baselines remain negative under the same raw certificate, even when other diagnostics appear healthy.

Two directions follow naturally. First, \(G_\tau\) can be used as a training-time guard or scheduler, so that intervention happens before the raw latent pathway crosses the boundary. Second, stronger self-supervised or pretrained teachers can make the same certificate more informative about partial or semantic collapse. A further extension would add a reconstruction-side target, such as \(T(\hat{x})\), and test consistency among \(T(x)\), \(S(z)\), and \(T(\hat{x})\) to expose decoder paths that reconstruct well while bypassing the raw witness. These are not guarantees established here; they are next steps toward moving from endpoint certification to trajectory-level control.

\appendix
\clearpage
\section*{Appendix}

\section*{Overview}
The appendix collects notation, threshold-sensitivity checks, secondary diagnostics, implementation details, additional empirical results, and the rescue-gradient calculation. These materials support the main certificate and experimental claims while keeping the main text focused on the theorem, the protocol, and the primary CIFAR-100 and Tiny-ImageNet-200 results.

\section{Notation}
\begin{table}[H]
\centering
\small
\caption{Main notation used in the paper.}
\label{tab:notation}
\begin{tabular}{ll}
\toprule
Symbol & Meaning \\
\midrule
\(T_x=T(\cdot\mid x)\) & teacher assignment for input \(x\) \\
\(\bar T\) & dataset-average teacher assignment \\
\(\Delta_K\) & \(K\)-simplex of categorical distributions \\
\(I_T(X;T)\) & teacher mutual information; exact best constant-student baseline \\
\(S^{raw}\) & raw student head that depends only on latent \(z\) \\
\(L^{raw}_{\mathrm{align}}\) & teacher-to-raw-student alignment loss \\
\(G_\tau\) & safety margin \(I_T-L^{raw}_{\mathrm{align}}-\tau\); positive means certificate holds \\
\(T_0(x)\) & cached fixed teacher target used in fixed-target experiments \\
\(G^{T_0}_\tau\) & fixed-target safety margin \(I(T_0)-L_{\mathrm{align}}(T_0,S)-\tau\) \\
\bottomrule
\end{tabular}
\end{table}

\section{Sensitivity to the Safety Threshold \texorpdfstring{\(\tau\)}{tau}}
\label{app:tau_sensitivity}
The safety margin is \(G_\tau=\widehat I_T-L^{raw}_{\mathrm{align}}-\tau\) in finite-sample reports. Changing \(\tau\) therefore does not require retraining; it shifts the reported margin by the chosen safety buffer. Table~\ref{tab:tau_sensitivity} reports representative endpoint margins for \(\tau\in\{0.05,0.10,0.20\}\). The main signs are stable for full and rescue runs. The Tiny no-alignment row is intentionally near the fixed-target constant boundary, and the CIFAR-100 seed-2 stress case is the most boundary-sensitive positive case.
\begin{table}[H]
\centering
\small
\caption{Sensitivity of representative endpoint margins to the safety threshold \(\tau\). Values are \(G_\tau\); positive values indicate that the certificate holds.}
\label{tab:tau_sensitivity}
\begin{tabular}{lccc}
\toprule
Endpoint & \(\tau=0.05\) & \(\tau=0.10\) & \(\tau=0.20\) \\
\midrule
CIFAR-100 full, strict-feasible seeds, min & 0.6042 & 0.5542 & 0.4542 \\
CIFAR-100 rescue, strict-feasible seeds, min & 0.5888 & 0.5388 & 0.4388 \\
CIFAR-100 seed 2 full, stress case & 0.1789 & 0.1289 & 0.0289 \\
CIFAR-100 seed 2 rescue, stress case & 0.3309 & 0.2809 & 0.1809 \\
Tiny-ImageNet-200 full, mean fixed-\(T_0\) & 5.0749 & 5.0249 & 4.9249 \\
Tiny-ImageNet-200 no-align, mean fixed-\(T_0\) & 0.0348 & -0.0152 & -0.1152 \\
Tiny-ImageNet-200 rescue, mean fixed-\(T_0\) & 5.0767 & 5.0267 & 4.9267 \\
\bottomrule
\end{tabular}
\end{table}

\section{Secondary Diagnostics}
\label{app:quality_diagnostics}
The main text focuses on certificate metrics. Table~\ref{tab:secondary_quality} records reconstruction and capacity diagnostics that are useful for interpretation but are not certificates.
\begin{table}[H]
\centering
\scriptsize
\caption{Secondary quality diagnostics. PSNR, probes, and active units can help interpret a run, but they do not replace the raw teacher-relative certificate.}
\label{tab:secondary_quality}
\begin{tabular}{llcc}
\toprule
Setting & Variant / method & PSNR & Other diagnostic \\
\midrule
CIFAR-100 seed 0 & Full / NoAlign / Rescue & 2.49 / 2.60 / 2.55 & -- \\
CIFAR-100 seed 1 & Full / NoAlign / Rescue & 2.50 / 2.60 / 2.54 & -- \\
CIFAR-100 seed 2 & Full / NoAlign / Rescue & 2.54 / 2.60 / 2.57 & stress case \\
Tiny-ImageNet-200 & Full / NoAlign / Rescue & 15.26 / 16.71 / 15.65 & no-align has highest PSNR but near-zero MI \\
Vanilla VAE & baseline & 19.79 & 64 active units \\
$\beta$-VAE ($10^{-5}$) & baseline & 21.23 & 64 active units \\
$\beta$-VAE ($10^{-3}$) & baseline & 16.15 & 21 active units \\
KL annealing & baseline & 16.17 & 19 active units \\
Free bits & baseline & 19.79 & 64 active units \\
Minimal GMVAE & baseline & 19.81 & 64 active units \\
VQ-VAE & baseline & 14.72 & 64 active units \\
\bottomrule
\end{tabular}
\end{table}

\section{Related-Work Positioning}
Posterior collapse has been a central concern in VAE research since the sequence and lossy-decoder settings where powerful decoders can bypass the latent pathway almost entirely \citep{bowman2016,chen2017}. Later work sharpened the optimization side of the story, showing that inference lag and related training dynamics can exacerbate collapse even when informative latents are in principle available \citep{he2019lagging}. More generally, a large body of work treats collapse as a mismatch between the ELBO objective and the desired information flow through the latent variables \citep{alemi2018,zhao2019}.

Existing mitigation strategies mostly operate by softening KL pressure, explicitly encouraging information flow, or changing the latent parameterization. Examples include broken-ELBO analyses and related free-bits thinking \citep{alemi2018}, InfoVAE-style objectives that directly favor informative latents \citep{zhao2019}, and discrete or mixture-structured alternatives such as GMVAE and VQ-VAE that make nontrivial assignments easier to sustain \citep{dilokthanakul2016,oord2017}. These methods are valuable, but they typically aim at improving latent usefulness in a broad sense rather than certifying exclusion of one specific structural collapse mode.

Our paper differs in both scope and proof target. We target a different proof object from general high-mutual-information or semantic-usefulness guarantees. Instead, we isolate a structurally identifiable mode: input-independent constant collapse. The theorem object is therefore not a teacher-aware classifier, but a strictly z-only head whose constancy is directly implied by constant latent codes. Distillation is relevant only in spirit: the teacher provides an input-dependent reference family, while the z-only student head is the certified witness \citep{hinton2015}. Accordingly, the main technical object is an exact constant baseline rather than a general lower bound on latent information.

\section{Experimental Protocol and Implementation}
\label{app:experimental_protocol}
\subsection{Experimental questions}
The experiments ask four concrete questions:
\begin{enumerate}
\item Does full teacher-guided training produce a strongly positive raw theorem margin?
\item Does removing explicit alignment collapse the raw z-only witness even when reconstruction remains usable?
\item Why is the practical safety margin $\Gtau$ needed beyond the bare theorem inequality?
\item Can conditional shared-safe protection rescue a run that has already collapsed under no-alignment training?
\end{enumerate}
Probe accuracy and PSNR are reported only as secondary diagnostics.

\subsection{Datasets and Metrics}
We report completed runs on CIFAR-10 and CIFAR-100. The key metrics are: teacher mutual information $I_T$, raw alignment $L^{raw}_{\mathrm{align}}$, bare raw margin $I_T-L^{raw}_{\mathrm{align}}$, practical margin $\Gtau$, student mutual information, linear-probe accuracy, and PSNR. Following the broader anti-collapse literature, we treat linear probes and reconstruction quality only as secondary diagnostics rather than theorem objects \citep{bowman2016,chen2017,zhao2019}. For CIFAR-10 we also include a rescue run that starts from a collapsed no-alignment checkpoint and then turns on conditional shared-safe protection.

\subsection{Exact derivation of the constant baseline}
For any constant $\alpha\in\Delta_K$,
\begin{align}
\E_x\KL(T_x\|\alpha)
&=\E_x \sum_c T_x(c)\log \frac{T_x(c)}{\alpha(c)} \\
&=\E_x \sum_c T_x(c)\log \frac{T_x(c)}{\bar T(c)} + \sum_c \bar T(c)\log \frac{\bar T(c)}{\alpha(c)} \\
&= I_T(X;\Tc)+\KL(\bar T\|\alpha).
\end{align}
Hence the minimum constant alignment cost is exactly $I_T(X;\Tc)$.

\subsection{Why raw and used quantities are both reported}
The theorem object is always the raw z-only head. Practical training, however, may use either the raw head or a constancy-preserving shared-safe calibrated head. We therefore record both raw and used quantities:
\begin{itemize}
\item raw alignment and raw margins for theorem reporting;
\item used alignment for the actual optimization objective; and
\item guard-set raw margins for optional rollback/backtracking tests.
\end{itemize}
This separation distinguishes theorem reporting from optimization details.

\subsection{Empirical Scope}
The empirical scope is chosen to test the certificate rather than to provide a broad benchmark study. The main CIFAR-100 result is a per-seed teacher-search stress test: each seed constructs its own teacher and then runs full, no-alignment, and rescue under that seed-specific fixed teacher. Two CIFAR-100 seeds pass the strict empirical teacher-feasibility filter, while one seed fails the strict minimum-mass criterion but still exhibits the same prevention--collapse--rescue certificate pattern. We therefore report this seed separately as a teacher-feasibility stress case rather than pooling it into a strict-feasible mean. The appendix also includes CIFAR-100 baseline diagnostics, Tiny-ImageNet-200 fixed-$\Tzero$ three-seed results with independently searched teachers, CIFAR-10 sanity checks, and a CIFAR-100 fixed-teacher control. Rollback and backtracking guards are left as training-side extensions; they are not part of the main reported objective.

\subsection{Training Protocol Summary}
The training protocol has four stages:
\begin{enumerate}
\item \textbf{Warm-up:} train the backbone before teacher fitting.
\item \textbf{Teacher construction and selection:} fit candidate teachers on warm-up features and choose one using informativeness, confidence, sharpness, balance, and coverage diagnostics.
\item \textbf{Main teacher-guided training:} optimize the decoder with $(z,\Tc)$ while monitoring the raw z-only alignment margin.
\item \textbf{Rescue and optional protection:} for the main evidence chain, initialize rescue from a collapsed no-alignment checkpoint and continue with the plain full objective. Shared-safe and rollback are optional training-side controls, not part of the main reported objective.
\end{enumerate}
This staging separates theorem reporting, which uses raw \(z\)-only quantities, from practical optimization details such as teacher quality control and optional protection.

\subsection{Detailed System Overview}
\label{app:detailed_system_overview}
Figure~\ref{fig:arch_appendix} provides a more detailed appendix-level system view, mainly to separate the staged workflow from the main-text architectural summary. As in the main figure, the teacher path is a fixed reference path that produces the assignment vector $T(x)$; in the fixed-$\Tzero$ setting the cached target is not a penalty tying $\mu_\theta(x)$ to $\mu_0(x)$. The lower-right branch in the figure shows optional training-side mechanisms only. The main CIFAR-100 and Tiny-ImageNet-200 full/rescue results instantiate the raw-head/plain-objective path rather than the shared-safe or rollback path.

\begin{figure}[H]
  \centering
  \includegraphics[width=0.98\linewidth]{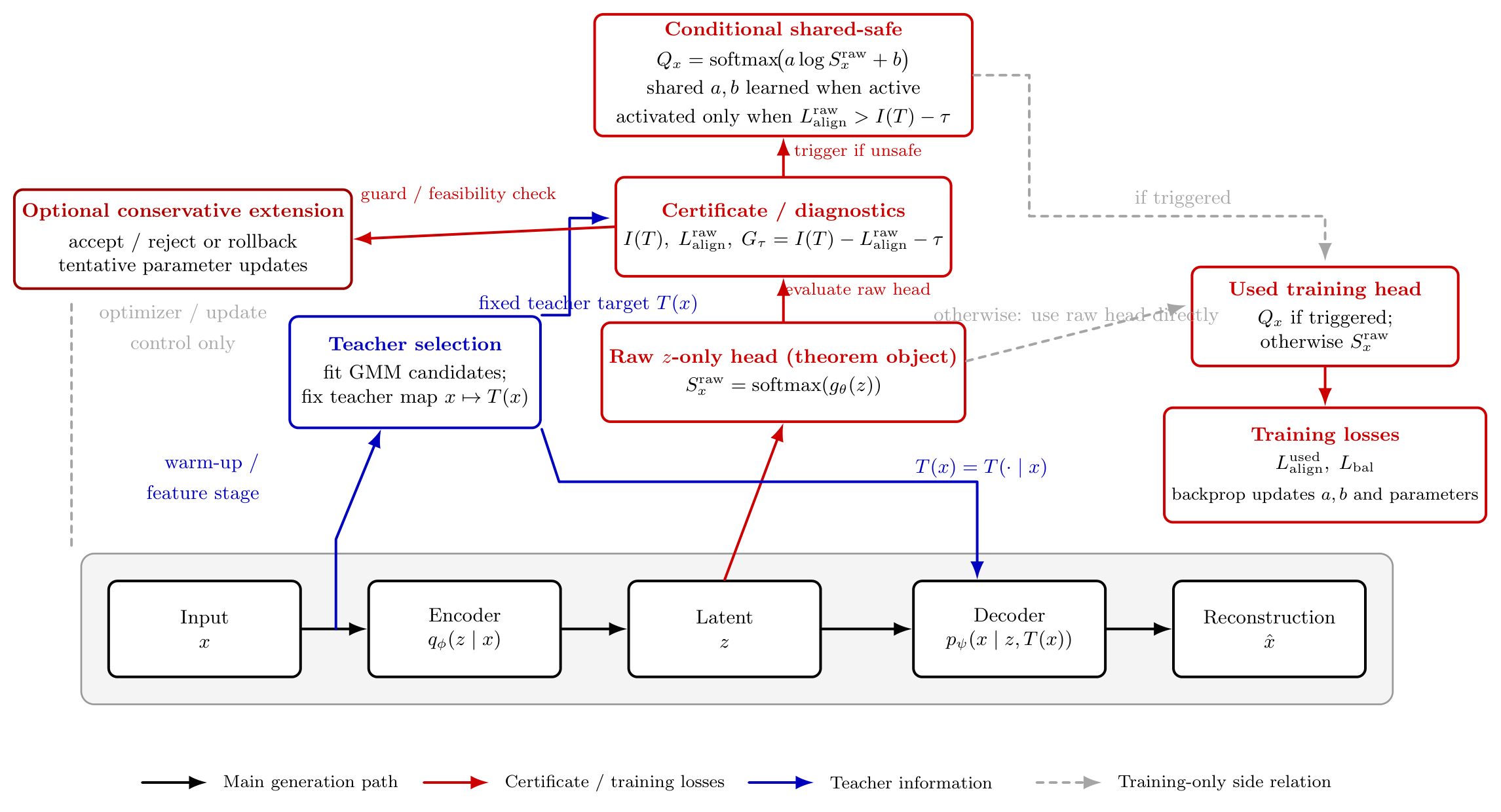}
  \caption[Detailed system overview.]{Detailed system overview. The appendix figure makes the staged workflow explicit: warm-up provides features for teacher fitting, the selected teacher map $x\mapsto T(x)$ is then fixed for the main stage, and certification remains tied to the raw z-only head $\Sc^{raw}$ even though generation may use $(z,T(x))$. Conditional shared-safe, used-head projection, and rollback/backtracking are shown only as optional training-side extensions. They are not used in the main CIFAR-100 or Tiny-ImageNet-200 evidence chain.}
  \label{fig:arch_appendix}
\end{figure}

\begin{table}[H]
\centering
\small
\caption{Certification objects versus optimization aids. The theorem is evaluated only from the certificate objects; optimization aids may help training but are not proof objects.}
\label{tab:cert_opt_aids}
\begin{tabular}{p{0.30\linewidth}p{0.62\linewidth}}
\toprule
Role & Objects \\
\midrule
Certificate objects & raw z-only witness \(S^{raw}\), fixed teacher target \(T\) or \(T_0\), \(I_T\), \(L^{raw}_{\mathrm{align}}\), \(G_\tau\), \(G^{T_0}_\tau\) \\
Optimization aids & decoder access to \(T(x)\), balance loss, optional shared-safe/used head, optional rollback or guard logic \\
\bottomrule
\end{tabular}
\end{table}

\subsection{Implementation Details}
The reported experiments use the same reporting convention across runs:
\begin{itemize}
\item The theorem quantity is always computed from the raw z-only head.
\item The practical safety margin uses $\tau=0.1$.
\item Shared-safe projection is optional and is not used in the main CIFAR-100 or Tiny-ImageNet-200 evidence chain.
\item The main full and rescue runs use the plain full objective; rescue differs from full only by initialization from the no-alignment collapsed checkpoint.
\item The CIFAR-10 rescue run starts from the completed no-alignment checkpoint and reuses the corresponding fixed teacher.
\item Probe accuracy and PSNR are included only as secondary diagnostics and are not used to define certification.
\end{itemize}

\subsection{Reproducibility Summary}
\label{app:repro_summary}
Table~\ref{tab:repro_summary} records the run-level conventions used in the reported experiments. The goal is not to duplicate every command-line flag, but to make the full/no-alignment/rescue/fixed-$T_0$ protocol explicit enough for checkpoints, teacher caches, and metric summaries to be interpreted consistently.

\begin{table}[H]
\centering
\small
\caption{Compact reproducibility summary for the reported experiments.}
\label{tab:repro_summary}
\setlength{\tabcolsep}{4pt}
\begin{tabular}{p{0.28\linewidth}p{0.64\linewidth}}
\toprule
Item & Setting / convention \\
\midrule
Backbone sharing & Full, no-alignment, and rescue use the same encoder, decoder, and raw z-only assignment head. \\
Teacher construction & Teacher candidates are fitted after warm-up on warm-up features and scored by teacher MI, margin, usage balance, and minimum component mass. \\
Teacher usage & CIFAR-100 per-seed teacher search constructs a teacher independently for each seed; Tiny-ImageNet-200 also re-searches a teacher independently for each seed before caching $T_0(x)$. \\
Training modes & Full enables alignment; no-alignment disables it; rescue resumes from a completed no-alignment checkpoint with alignment enabled; fixed-$T_0$ uses cached pseudo-targets. \\
Safety threshold & The practical margin uses $\tau=0.1$ throughout the reported experiments. \\
Shared-safe trigger & The used-head protection is activated only when $L^{raw}_{\mathrm{align}}>I_T-\tau$; otherwise the raw head is optimized directly. \\
Reporting rule & All certificate values are computed from the raw z-only head; used-head quantities are optimization helpers only. \\
Secondary diagnostics & Student MI, linear probe, PSNR, and active units are reported as diagnostics and never substituted for the certificate. \\
\bottomrule
\end{tabular}
\end{table}

\subsection{Reproducibility Artifacts}
\label{app:repro_artifacts}
The reproducibility package should include an anonymized code archive or repository snapshot with a README. The README should specify installation, datasets, teacher search and caching, CIFAR-100 full/no-alignment/rescue commands, Tiny-ImageNet-200 fixed-$T_0$ commands, baseline commands, and figure-generation scripts. For anonymized review, commands should be provided as scripts that expose teacher search, caching, full training, no-alignment, rescue, baseline evaluation, and figure generation without revealing author identity.

\subsection{Existing Assets and Licenses}
\label{app:asset-licenses}

The experiments use standard public academic benchmarks and open-source software packages. We do not redistribute the raw datasets in the supplementary artifact. The released code expects users to obtain the datasets from their original sources or through standard library download utilities, and to comply with the corresponding dataset terms of use.

\begin{table}[h]
\centering
\small
\begin{tabular}{p{0.22\linewidth}p{0.47\linewidth}p{0.23\linewidth}}
\toprule
Asset & Use in this work & Source / license notes \\
\midrule
CIFAR-10 & Sanity-check experiments and development runs. & Standard public CIFAR benchmark; users should obtain it from the original CIFAR source or torchvision and follow the original dataset terms. \\
CIFAR-100 & Main teacher-search, rescue, and baseline experiments. & Standard public CIFAR benchmark; users should obtain it from the original CIFAR source or torchvision and follow the original dataset terms. \\
Tiny-ImageNet-200 & Fixed-\(T_0\) rescue experiments. & Public academic benchmark derived from ImageNet-style data; users should obtain it from the original Tiny-ImageNet distribution and comply with the applicable dataset terms. \\
PyTorch / torchvision & Model training, data loading, and neural-network components. & Open-source Python packages; users should follow the licenses distributed with the installed packages. \\
NumPy, pandas, scikit-learn, matplotlib & Numerical computation, teacher fitting, summaries, and plotting. & Open-source Python packages; users should follow the licenses distributed with the installed packages. \\
\bottomrule
\end{tabular}
\caption{Existing datasets and software assets used in the experiments. The supplementary artifact contains code and scripts only; it does not redistribute raw benchmark data.}
\label{tab:asset-licenses}
\end{table}

The supplementary code package is released only for reproducing the experiments reported in this paper. It contains training scripts, smoke tests, baseline scripts, plotting utilities, and configuration files. It does not include model checkpoints intended for deployment, pretrained generative models, private data, or scraped datasets.

\subsection{Compute environment}
The experiments were run in a Linux container on an AutoDL-style host. The shell reported Python~3.12.3 and pip~24.0. The host CPU was an Intel Xeon Platinum 8260 class machine, while the GPU pool exposed to the container contained eight NVIDIA GeForce RTX 4080 SUPER devices with NVIDIA driver 580.105.08. The Python environment used PyTorch 2.8.0+cu128 and torchvision 0.23.0+cu128, with NumPy 2.3.2, scikit-learn 1.8.0, pandas 3.0.1, and matplotlib 3.10.5. This paper does not claim distributed multi-GPU training; experiments were launched as separate runs from saved scripts and logs.

\begin{table}[H]
\centering
\small
\caption{Compact compute and software environment summary.}
\label{tab:compute_env}
\setlength{\tabcolsep}{4pt}
\begin{tabular}{p{0.30\linewidth}p{0.62\linewidth}}
\toprule
Item & Recorded environment \\
\midrule
CPU host & Intel Xeon Platinum 8260 CPU @ 2.40GHz; 96 logical CPUs visible to the host; the runtime shell reported \texttt{nproc=1}. \\
GPU pool & Eight NVIDIA GeForce RTX 4080 SUPER devices visible through \texttt{/proc/driver/nvidia/gpus}. \\
NVIDIA driver & 580.105.08. \\
Python / pip & Python 3.12.3; pip 24.0. \\
Core ML stack & PyTorch 2.8.0+cu128; torchvision 0.23.0+cu128; CUDA 12.8 userland packages. \\
Scientific stack & NumPy 2.3.2; SciPy 1.17.1; scikit-learn 1.8.0; pandas 3.0.1; matplotlib 3.10.5. \\
\bottomrule
\end{tabular}
\end{table}

\subsection{Teacher Selection for CIFAR-100}
The selected CIFAR-100 teacher in the reported paired-teacher package has \texttt{feasible=true} under the strict minimum-component-mass filter. This does not change the theorem itself, which only requires a fixed nonconstant teacher family, but it does strengthen the practical teacher-selection story relative to earlier teacher choices.

\begin{figure}[H]
  \centering
  \includegraphics[width=0.78\linewidth]{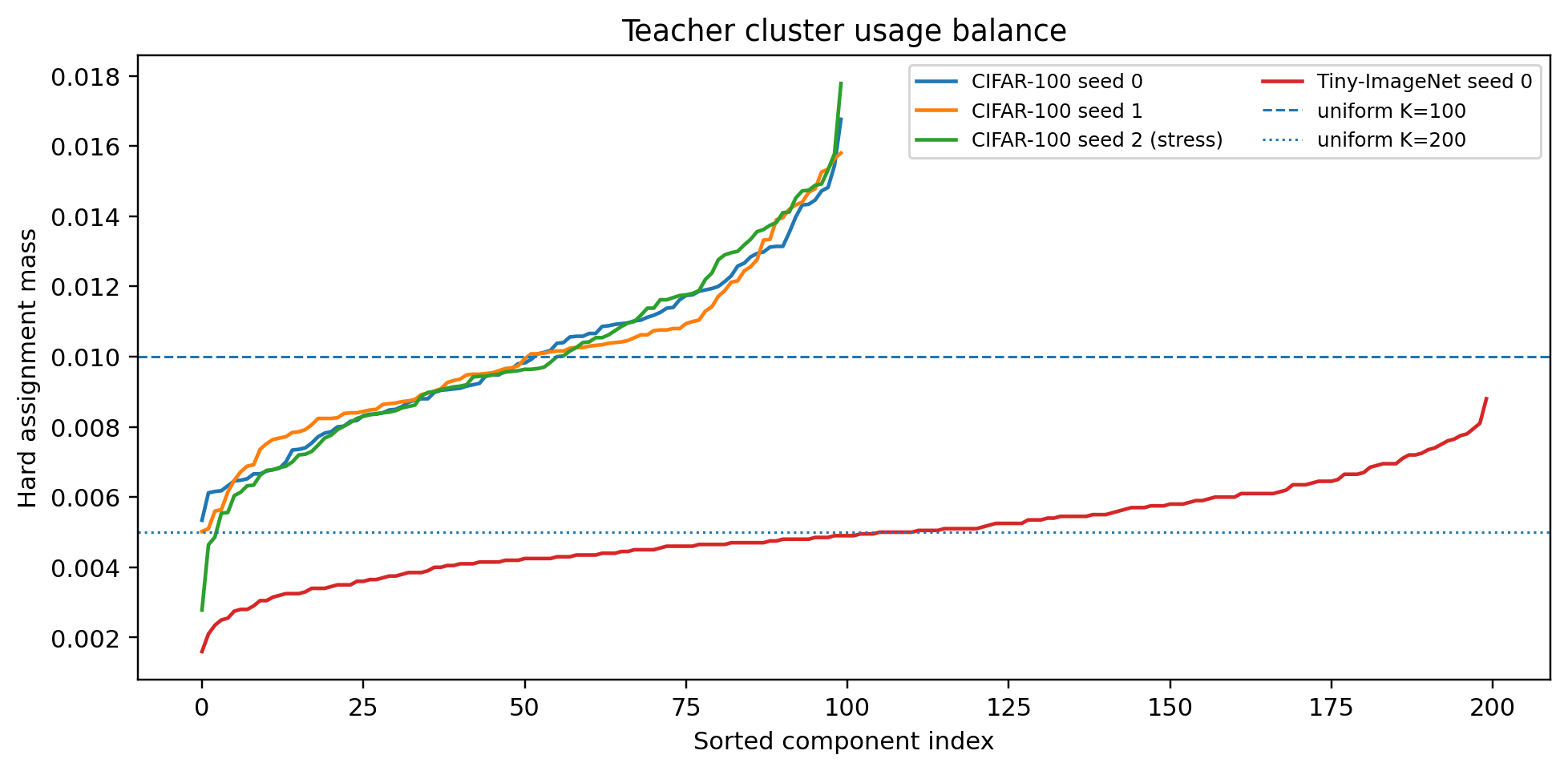}
  \caption{Teacher cluster usage balance for the selected CIFAR-100 teacher and the Tiny-ImageNet-200 seed-0 teacher. Counts are normalized and sorted for readability. The selected teachers are not perfectly uniform, but their hard and soft usage diagnostics avoid degenerate single-cluster targets.}
  \label{fig:teacher_usage_balance}
\end{figure}

\section{Additional Empirical Results}
\subsection{CIFAR-10 Sanity-Check Endpoints}
CIFAR-10 is not part of the main evidence chain, but it remains a useful small-scale sanity check. The full/no-alignment/rescue pattern is numerically clean and supports the mechanism on a simpler benchmark.
\begin{table}[H]
\centering
\small
\caption{CIFAR-10 sanity-check endpoints. CIFAR-10 is reported as a small-scale mechanism check rather than a main evidence block.}
\label{tab:c10_sanity}
\setlength{\tabcolsep}{4pt}
\resizebox{\linewidth}{!}{%
\begin{tabular}{lcccccc}
\toprule
Variant & $I_T$ & Align & Bare margin & Student MI & PSNR & Probe \\
\midrule
Full & 1.9127 & 0.0385 & 1.8742 & 0.7081 & 2.2758 & 0.3712 \\
No alignment & 1.9130 & 1.9436 & -0.0306 & $2.86\times10^{-7}$ & 2.4543 & 0.3860 \\
Rescue from noalign & 1.6668 & 0.0182 & 1.6486 & 0.9861 & 2.3829 & 0.3624 \\
\bottomrule
\end{tabular}}
\end{table}

\subsection{CIFAR-100 Fixed-Teacher Control}
The older CIFAR-100 fixed-teacher run is retained as a control experiment. It isolates optimization dynamics under one selected feasible teacher: the teacher is held fixed, while training seeds change. This result should not be pooled with the per-seed teacher-search protocol in the main text.

\begin{table}[H]
\centering
\scriptsize
\caption{CIFAR-100 fixed-teacher training-seed control. This table is a controlled optimization-dynamics check under one selected feasible teacher, not a per-seed teacher-construction robustness result.}
\label{tab:c100_fixed_teacher_appendix}
\setlength{\tabcolsep}{3pt}
\resizebox{\linewidth}{!}{%
\begin{tabular}{lllccccc}
\toprule
Seed & Variant & Status & $I_T^{cur}$ & Align & $\Gtau$ & Student MI & PSNR \\
\midrule
0 & Full & final & 0.3624 & 0.0223 & 0.2401 & 0.3201 & 2.5275 \\
0 & NoAlign & final & 0.0686 & 4.6052 & -4.6366 & $6.91\times10^{-6}$ & 2.6029 \\
0 & Rescue & final & 0.5668 & 0.0274 & 0.4393 & 0.5583 & 2.5657 \\
1 & Full & final & 0.6472 & 0.0262 & 0.5210 & 0.6328 & 2.5357 \\
1 & NoAlign & final & 0.1514 & 4.6052 & -4.5537 & $7.83\times10^{-6}$ & 2.6031 \\
1 & Rescue & final & 0.6799 & 0.0274 & 0.5525 & 0.6337 & 2.5556 \\
2 & Full & final & 0.2287 & 0.0174 & 0.1113 & 0.2057 & 2.5486 \\
2 & NoAlign & final & 0.0560 & 4.6052 & -4.6492 & $1.07\times10^{-5}$ & 2.6045 \\
2 & Rescue & final & 0.5127 & 0.0144 & 0.3984 & 0.5076 & 2.5532 \\
\bottomrule
\end{tabular}}
\end{table}

\subsection{CIFAR-100 Teacher-Feasibility Stress Case}
The per-seed CIFAR-100 seed~2 teacher fails the strict minimum-mass feasibility filter, but it is not a degenerate teacher: it has $I_T=4.3342$, mean top-1 margin $0.8407$, high-margin fraction $0.9475$, hard-balance KL $0.0393$, and soft-usage KL $0.0388$. The failure is therefore a strict teacher-construction quality-control failure, not a failure of the constant-collapse certificate. As shown in Table~\ref{tab:cifar100-search}, full/no-alignment/rescue still follow the desired sign pattern.

\subsection{Additional Variant Checks}
Additional CIFAR-100 shared-safe and no-late-refinement variants were run under obsolete teacher choices and are not included in the main evidence chain. The reported CIFAR-100 claims use the per-seed teacher-search protocol; the seed~2 teacher that fails the strict feasibility filter is reported explicitly as a stress case rather than pooled into the strict-feasible mean.

\subsection{Scope of Additional Results}
The additional results focus on figures that directly illuminate the prevention--collapse--rescue pattern: rescue transitions, run-level training curves, and compact reproducibility summaries. CIFAR-100 per-seed curves, Tiny-ImageNet-200 fixed-$T_0$ curves, the fixed-teacher control, and baseline diagnostic curves are included below.

\subsection{Rescue and Training-Curve Panels}
The appendix figures emphasize the rescue pathway, the CIFAR-100 run-level trajectories, and the Tiny-ImageNet-200 fixed-$T_0$ trajectories, complementing the endpoint tables in the main paper.

\begin{figure}[H]
  \centering
  \includegraphics[width=0.82\linewidth]{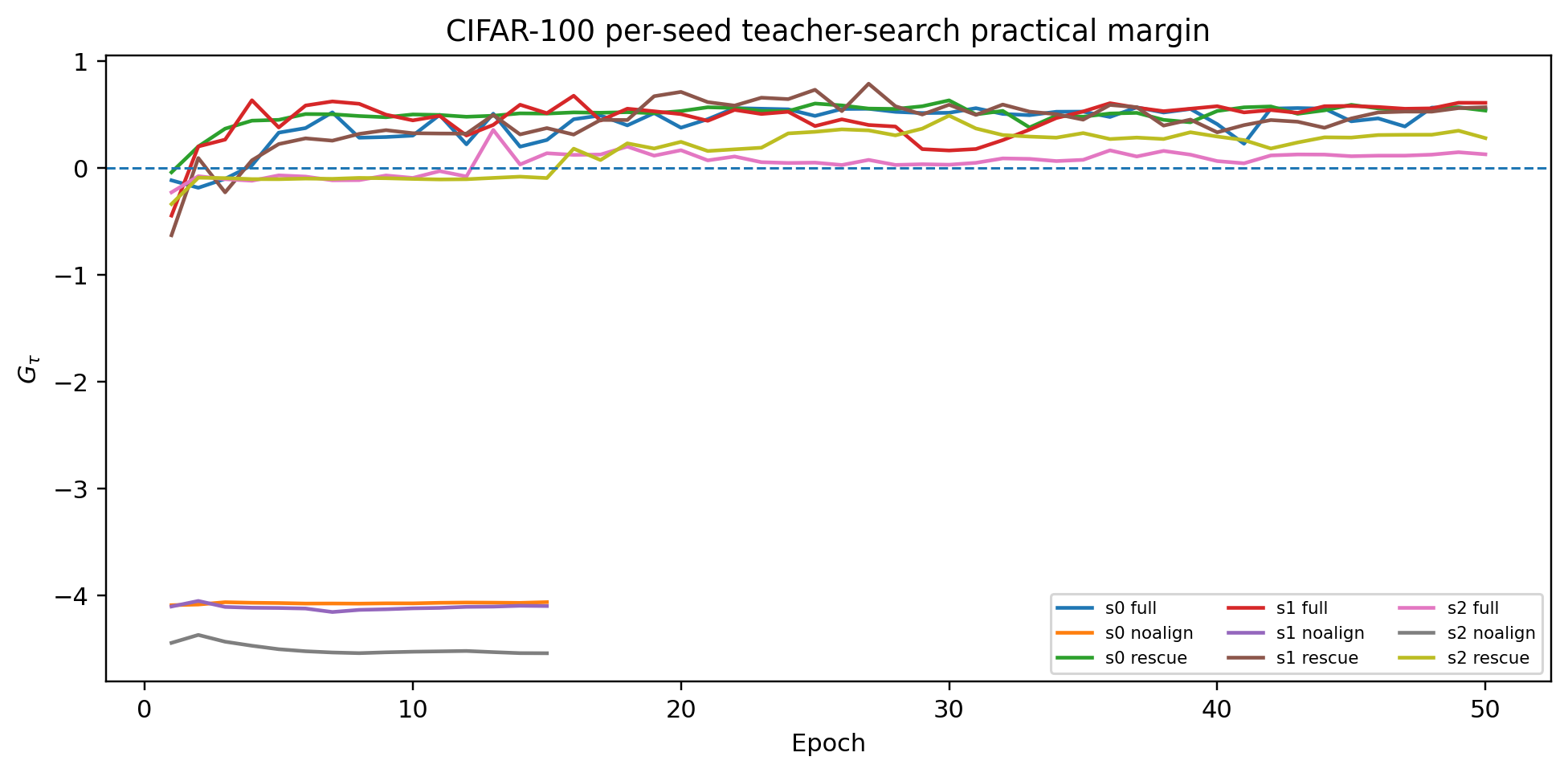}
  \caption{CIFAR-100 three-seed practical-margin trajectories. Full runs remain certificate-positive, no-alignment runs collapse to a strongly negative regime, and rescue restores positive margins in all completed seeds.}
  \label{fig:c100_3seed_gtau}
\end{figure}

\begin{figure}[H]
  \centering
  \includegraphics[width=0.82\linewidth]{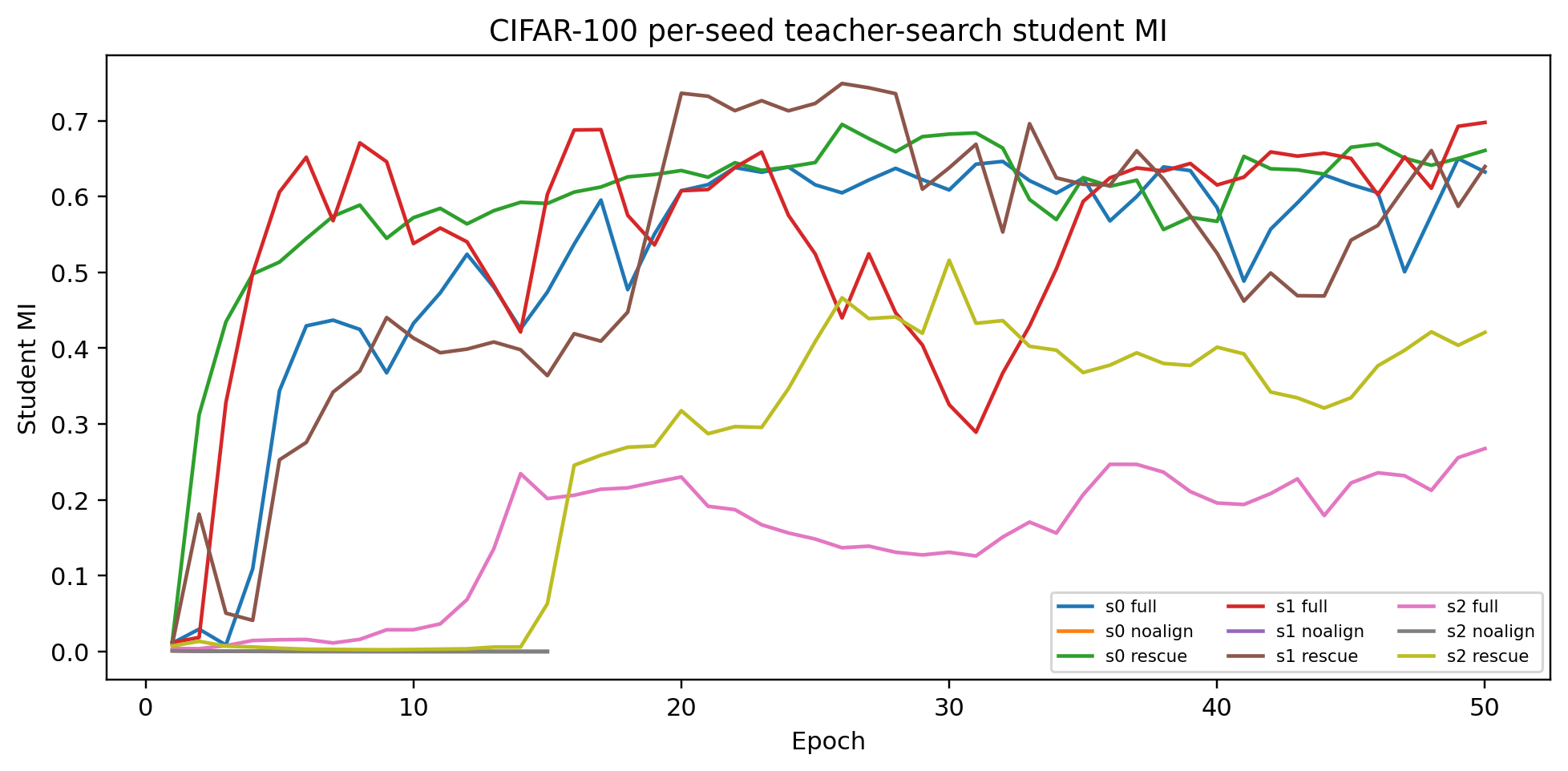}
  \caption{CIFAR-100 three-seed student-MI trajectories. The no-alignment runs approach nearly constant assignments, while full and rescue retain nontrivial student mutual information.}
  \label{fig:c100_3seed_mi}
\end{figure}

\begin{figure}[H]
  \centering
  \includegraphics[width=0.82\linewidth]{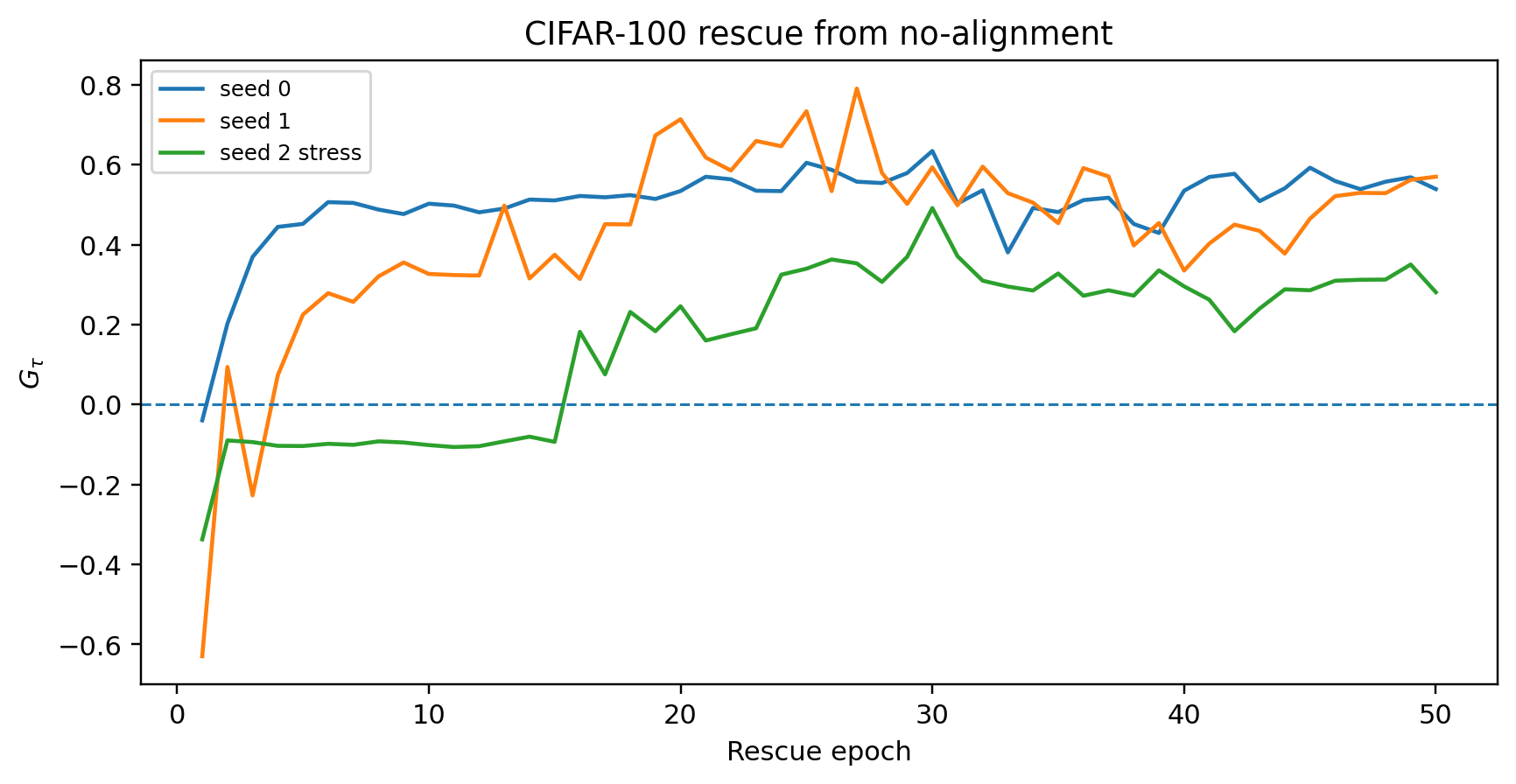}
  \caption{CIFAR-100 rescue practical-margin trajectories under the per-seed teacher-search protocol. Rescue starts from no-alignment collapsed checkpoints and recovers a positive raw-head certificate for all reported seeds.}
  \label{fig:c100_rescue_gtau}
\end{figure}

\begin{figure}[H]
  \centering
  \includegraphics[width=0.82\linewidth]{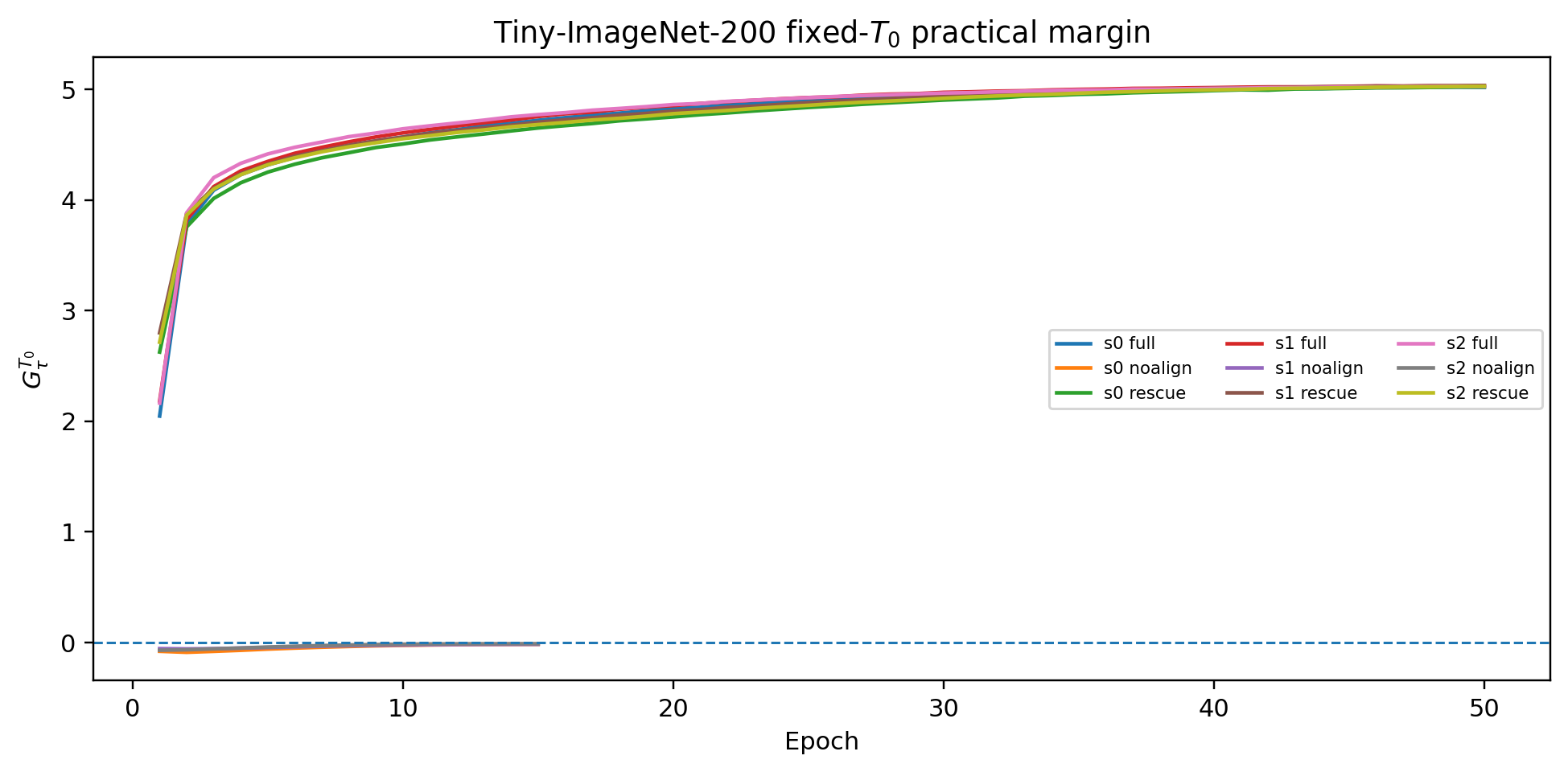}
  \caption{Tiny-ImageNet-200 fixed-$T_0$ three-seed certificate trajectories. Each seed re-searches a teacher before caching $T_0(x)$; full and rescue become strongly positive, while no-alignment remains near the fixed-target constant boundary.}
  \label{fig:tiny_3seed_gtau}
\end{figure}

\begin{figure}[H]
  \centering
  \includegraphics[width=0.82\linewidth]{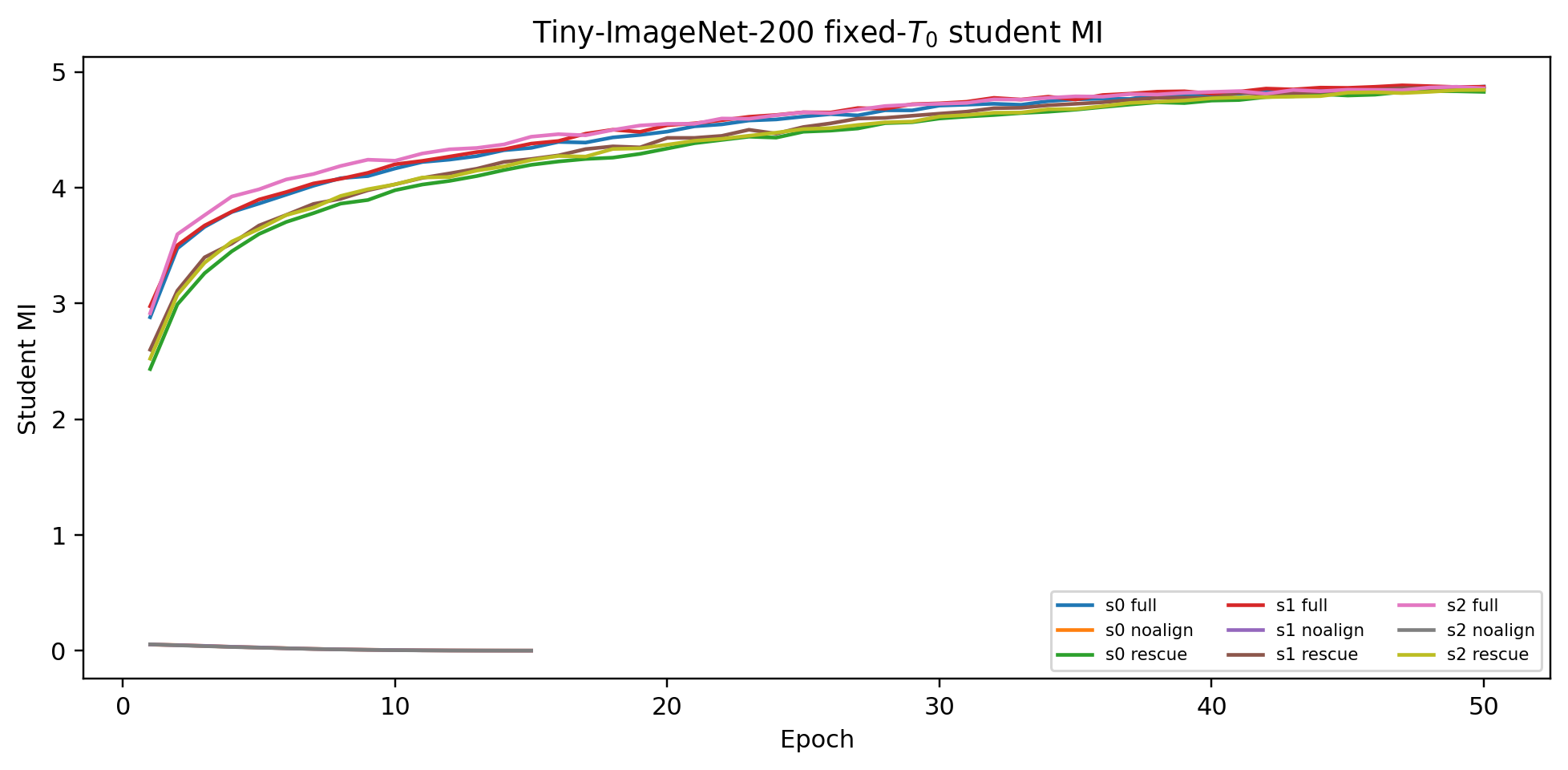}
  \caption{Tiny-ImageNet-200 three-seed student-MI trajectories. Full and rescue maintain high student MI relative to the fixed target, whereas no-alignment collapses to nearly input-independent assignments.}
  \label{fig:tiny_3seed_mi}
\end{figure}

\begin{figure}[H]
  \centering
  \includegraphics[width=0.82\linewidth]{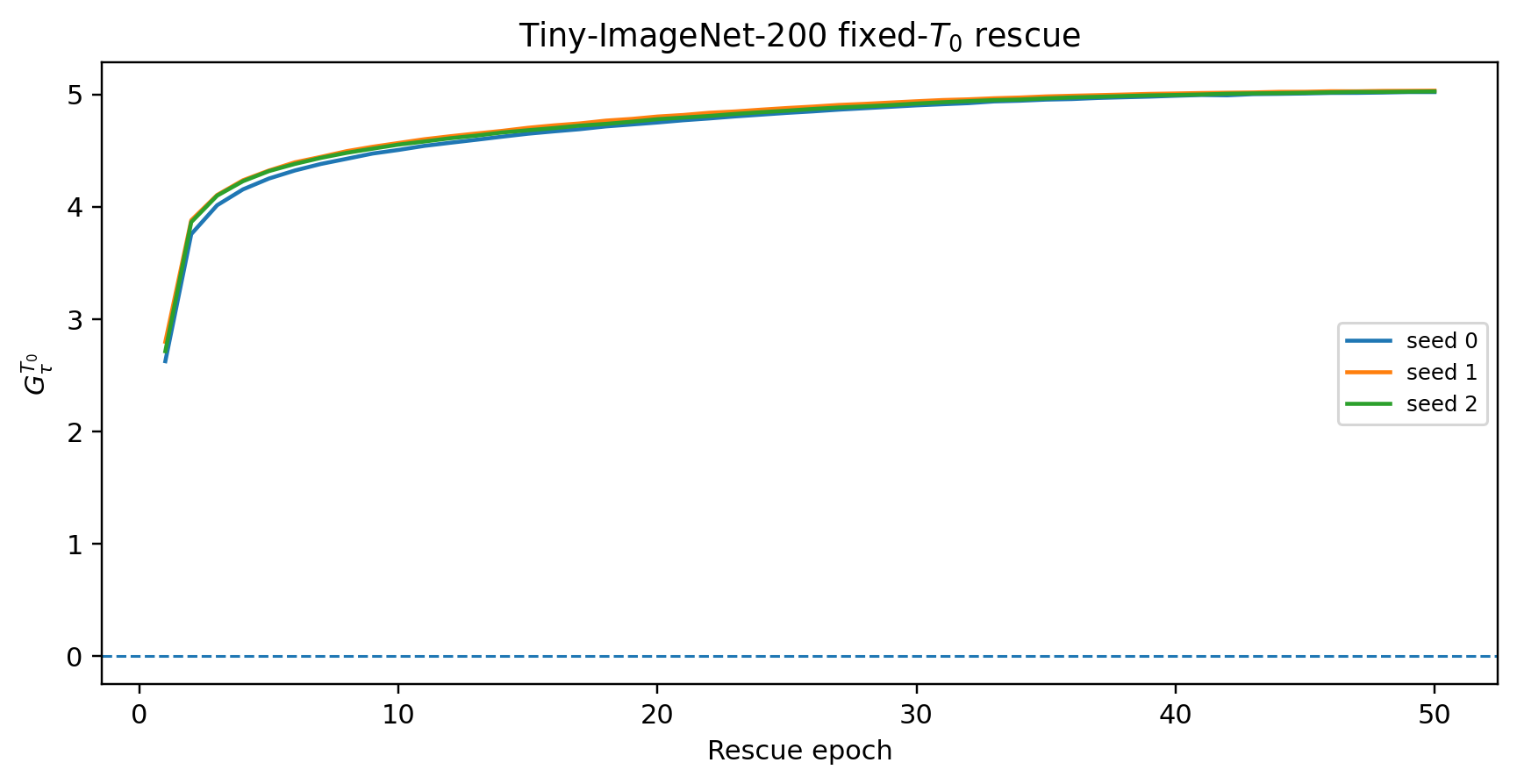}
  \caption{Tiny-ImageNet-200 fixed-$T_0$ rescue certificate trajectories. Reintroducing fixed-target alignment from collapsed no-alignment checkpoints restores large positive fixed-target margins.}
  \label{fig:tiny_rescue_gtau}
\end{figure}

\begin{figure}[H]
  \centering
  \includegraphics[width=0.82\linewidth]{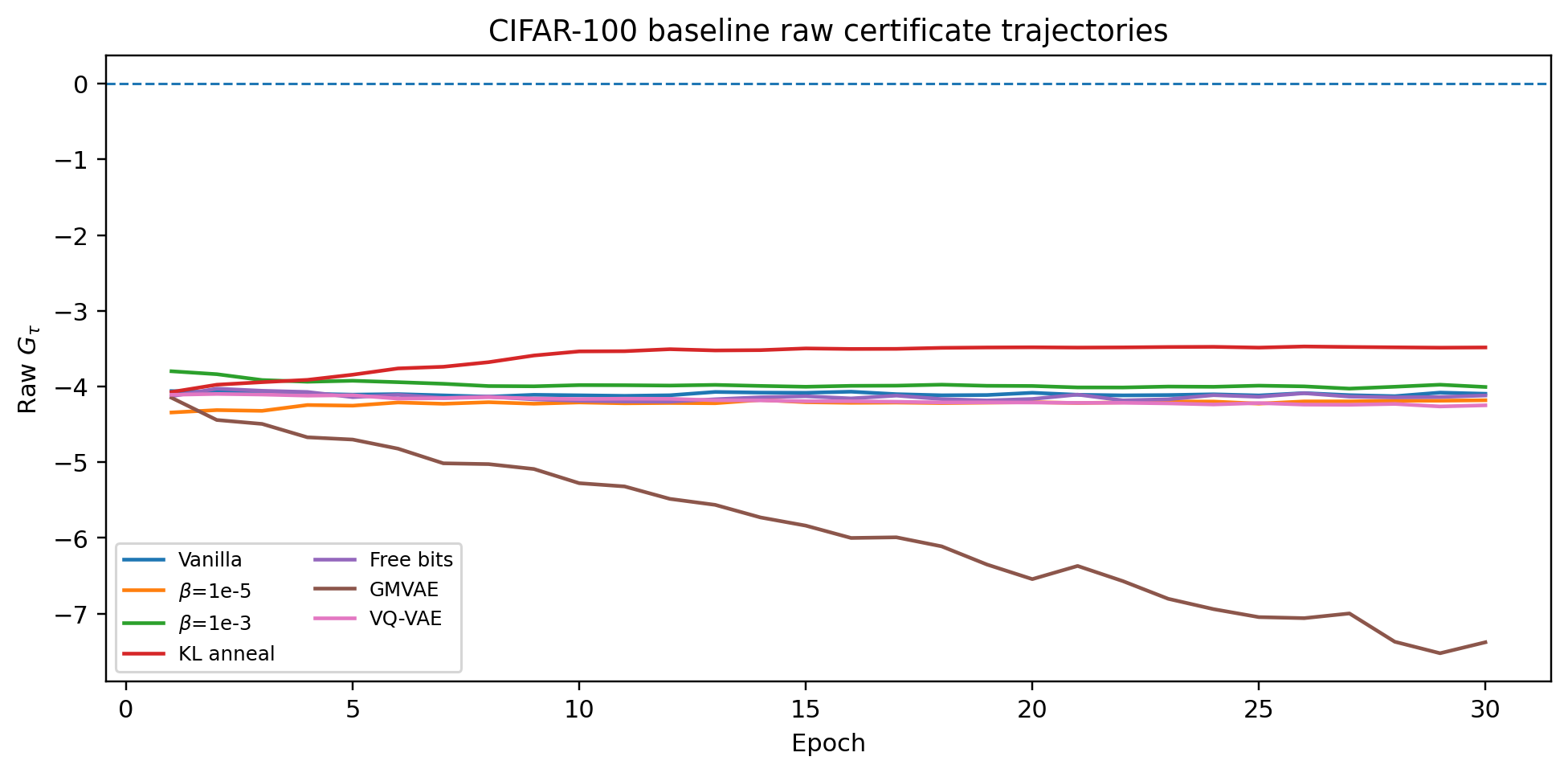}
  \caption{CIFAR-100 baseline raw-margin trajectories. The standard VAE-style baselines remain raw-certificate negative under the fixed teacher, supporting the distinction between reconstruction/active-unit diagnostics and the teacher-relative raw-head certificate.}
  \label{fig:c100_baseline_curves}
\end{figure}

\section{Optimization Dynamics Near Collapse}
\label{app:opt_dynamics_student_rescue}

This appendix analyzes optimization behavior near the constant-collapse boundary and formalizes when rescue can or cannot reactivate useful gradients. The analysis supports the main theorem and the empirical rescue results without adding a separate certificate.

\tightparagraph{Compression versus certified structure.}
The total objective combines reconstruction, KL regularization, and teacher-guided alignment:
\[
\mathcal L
=
\mathcal L_{\mathrm{recon}}
+\beta \mathcal L_{\mathrm{KL}}
+\lambda_{\mathrm{align}} L_{\mathrm{align}}^{\mathrm{raw}}
+\lambda_{\mathrm{bal}} L_{\mathrm{bal}}.
\]
KL regularization induces shrinkage toward a simple prior-compatible latent representation, while the raw alignment term rewards input-dependent latent structure that beats the constant baseline $I_T$. The resulting training dynamics are therefore governed by a competition between compression and certified discriminative structure. Non-collapse arises when the alignment gain outweighs the KL-driven collapse pressure.

\tightparagraph{The role of \texorpdfstring{$\lambda_{\mathrm{align}}$}{lambda\_align}.}
Increasing $\lambda_{\mathrm{align}}$ strengthens the optimization pressure to reduce the raw alignment loss without changing the direction of the alignment gradient itself. Near the collapse boundary, this typically enlarges the practical safety margin
\[
\Gtau = I_T - L_{\mathrm{align}}^{\mathrm{raw}} - \tau
\]
by pushing $L_{\mathrm{align}}^{\mathrm{raw}}$ further below the constant baseline. In this sense, $\lambda_{\mathrm{align}}$ controls the distance to collapse in certificate space rather than in a simple geometric parameter-space sense. When collapse is already deep, however, increasing $\lambda_{\mathrm{align}}$ alone may be insufficient, because the system may already be trapped in a bad basin dominated by KL shrinkage, decoder bypass, or vanishing encoder-side sensitivity.

\tightparagraph{From certificate to harder anti-collapse objectives.}
Once a certified safety margin $\Gtau$ is available, one may also define harder anti-collapse training variants around the certificate boundary. For example, one may consider a barrier-style objective
\[
\mathcal L = \mathcal L_{\mathrm{base}} - \mu \log(\Gtau - m),
\]
which makes the collapse boundary behave like an infinite wall as $\Gtau \to m^+$. We do not pursue this direction in the present paper, because our goal here is to certify and analyze the constant-collapse boundary itself rather than to optimize a new barrier-based training procedure.

\tightparagraph{Near-collapse drift and collapse probability.}
Under a guarded $(\beta,\lambda_{\mathrm{align}})$ schedule and a non-vanishing alignment-gradient condition in the near-collapse strip, the practical margin $\Gtau$ admits positive drift away from the collapse boundary. Consequently, the probability of entering the collapse set can be made exponentially small in the initial safety margin and inverse noise level. This statement is a local dynamical claim near the practical boundary, not a global guarantee against all forms of posterior collapse.

\tightparagraph{Zero total gradient versus zero alignment gradient.}
A zero total gradient does not imply a zero alignment gradient. If the stationary point arises from cancellation between the VAE and alignment gradients, rather than from a vanishing alignment gradient itself, then changing $\lambda_{\mathrm{align}}$ immediately changes the total descent direction and may trigger rescue. The truly difficult case is stronger: all relevant parameter gradients vanish and the latent representation is collapsed.

\tightparagraph{Two training tracks and their combinations.}
From the optimization viewpoint, training can be read as the interaction of two coupled tracks. The standard VAE track is driven by reconstruction and KL regularization and tends to compress the latent representation, while the student/alignment track is driven by $L_{\mathrm{align}}$ and related balance terms and rewards input-dependent structure that beats the constant baseline. In practice these tracks may be run jointly, alternated in either order, or organized into a staged protocol such as student-only rescue followed by reactivated joint training. Near the collapse boundary, the ordering and the relative scale of $(\beta,\lambda_{\mathrm{align}})$ matter because the two tracks may either cancel each other or reinforce a drift back toward the safe region.

\section{Guarded Dynamics for Prevention and Rescue}
\label{app:guarded-dynamics}
The certificate can also be viewed as a training-time guard. This argument is not needed for the constant-student identity itself; it only gives a sufficient optimization condition for maintaining or recovering the certified margin.

Let
\[
L_a(\theta)=L^{raw}_{\mathrm{align}}(\theta)
=\mathbb E_x\KL(T_x\|S^{raw}_\theta(x)),
\qquad
G_\tau(\theta)=I_T(X;T)-L_a(\theta)-\tau .
\]
Write the training objective as
\[
\mathcal L(\theta)=\mathcal L_0(\theta)+\lambda(t)L_a(\theta),
\]
where \(\mathcal L_0\) contains reconstruction, latent KL, balance, and other non-alignment terms. Under the continuous-time gradient-flow idealization,
\[
\dot\theta
=
-\nabla\mathcal L_0(\theta)
-\lambda(t)\nabla L_a(\theta).
\]
Since \(G_\tau=I_T-L_a-\tau\), we have
\[
\dot G_\tau
=
-\langle\nabla L_a,\dot\theta\rangle
=
\langle\nabla L_a,\nabla\mathcal L_0\rangle
+
\lambda(t)\|\nabla L_a\|^2 .
\]

This gives a sufficient boundary condition. Suppose that, near the boundary \(G_\tau=0\), the effective alignment gradient is bounded away from zero,
\[
\|\nabla L_a(\theta)\|^2\ge m^2>0,
\]
and the drift from the remaining objective is bounded below,
\[
\langle\nabla L_a,\nabla\mathcal L_0\rangle\ge -b .
\]
Then any alignment weight satisfying
\[
\lambda(t)>b/m^2
\]
makes
\[
\dot G_\tau>0
\]
on the boundary. In this continuous-time approximation, the certified region
\[
\mathcal C_\tau=\{\theta:G_\tau(\theta)>0\}
\]
is forward invariant whenever these inequalities hold. Thus, if training starts with \(G_\tau>0\), the trajectory cannot cross into the uncertified side through such a boundary point.

The same calculation explains rescue. At a constant student \(S_x\equiv\bar T\), the KL-alignment gradient with respect to the student logits \(a_x\) is
\[
\nabla_{a_x}L_a=S_x-T_x=\bar T-T_x .
\]
Therefore the descent direction is
\[
-\nabla_{a_x}L_a=T_x-\bar T,
\]
which is input-dependent whenever the teacher is nonconstant. If the encoder and raw head retain a nonzero local response in this centered teacher direction, re-enabling alignment produces an input-dependent update and can move the raw witness away from the constant-student region.

A simple sufficient local-response condition is that there exists a small parameter perturbation \(\delta\theta\) whose induced logit change
\[
\delta a_x=J_\theta(x)\delta\theta
\]
is positively correlated with the centered teacher variation:
\[
\mathbb E_x
\left\langle
T_x-\bar T,\,
J_\theta(x)\delta\theta
\right\rangle
>0 .
\]
This condition can fail if the encoder has become fully input-independent, if the raw head has no useful sensitivity, or if the teacher variation is too weak to provide a usable direction. In such cases, increasing the alignment weight alone may not rescue the model.

The practical rescue protocol is therefore student-first in the following sense: once \(G_\tau\le 0\), training temporarily makes the raw alignment term dominant before returning to the full objective. Equivalently, one can view the schedule as
\[
\lambda(t)=
\begin{cases}
\lambda_{\mathrm{base}}, & G_\tau>\delta,\\
\lambda_{\mathrm{guard}}, & 0<G_\tau\le \delta,\\
\lambda_{\mathrm{rescue}}, & G_\tau\le 0,
\end{cases}
\qquad
\lambda_{\mathrm{rescue}}\ge \lambda_{\mathrm{guard}}\ge \lambda_{\mathrm{base}} .
\]
The guard increases alignment pressure before the raw witness crosses the constant-student boundary. Rescue makes the alignment term dominant after the boundary has been crossed, provided that the encoder and raw head still have enough local sensitivity to respond to the teacher signal.

This is a sufficient guarded-dynamics condition under a continuous-time gradient-flow idealization, not a global guarantee for all collapsed states. The theorem-level certificate remains the static implication \(G_\tau>0\Rightarrow\) nonconstant raw witness. The guarded-dynamics argument explains when training can maintain that certificate, and when reintroducing alignment can recover it from a collapsed checkpoint. The empirical rescue results show that the tested collapsed checkpoints remain inside such a rescue basin.

\tightparagraph{Modularity with respect to VAE variants.}
The teacher--student mechanism is modular with respect to the underlying VAE objective. It can be added to a standard VAE or to many VAE variants by augmenting the original training loss with a raw latent-only alignment term,
\[
\mathcal L
=
\mathcal L_{\mathrm{base\;VAE}}
+
\lambda_{\mathrm{align}}L^{raw}_{\mathrm{align}}
+
\lambda_{\mathrm{bal}}L_{\mathrm{bal}}.
\]
In this sense, the certificate is orthogonal to common architectural or objective-level anti-collapse techniques, such as KL schedules, free bits, decoder restrictions, or alternative latent structures. This orthogonality is modular rather than dynamical: the alignment term still changes the optimization trajectory through the encoder and the raw head. Its role is to provide a teacher-relative collapse boundary and a guard signal, not to replace the base VAE objective.

\tightparagraph{Compatibility with quality-oriented regularizers.}
The certificate is compatible with additional regularizers that target representation quality, reconstruction quality, or partial-collapse control. In general, one may augment the base objective as
\[
\mathcal L
=
\mathcal L_{\mathrm{base\;VAE}}
+
\lambda_{\mathrm{align}}L^{raw}_{\mathrm{align}}
+
\lambda_{\mathrm{qual}}L_{\mathrm{qual}},
\]
where \(L_{\mathrm{qual}}\) may include usage-balance penalties, active-unit or mutual-information terms, contrastive losses, reconstruction-side losses, or supervised or self-supervised probing objectives when such signals are available. These terms address properties that are different from the raw certificate: mode usage, linear readability, reconstruction fidelity, or semantic quality. They can therefore be added modularly, provided that the reported certificate remains computed from the raw \(z\)-only witness. This distinction is important: better PSNR, stronger linear-probe accuracy, or improved usage balance can improve the learned representation, but they do not by themselves replace the implication \(G_\tau>0\Rightarrow\) nonconstant raw witness.

\section{Semantic-Control Extensions}
\label{app:finite_semantic_control}

A broader interpretation of the framework is that \(T\) and \(S\) form a finite-dimensional semantic proxy layer for an otherwise high-dimensional generative training process. The teacher \(T(x)\) is a fixed, controllable approximation to the sample distribution in a semantic or self-supervised feature space, while the raw student \(S(z)\) is the latent-path approximation to this reference. This makes a normally invisible training failure visible through \(I_T\), \(L^{raw}_{\mathrm{align}}\), and \(G_\tau\).

This view also clarifies what is, and is not, certified. The present paper certifies the edge \(T(x)\to S(z)\): if the raw z-only witness beats the constant-student threshold, then the matched nonconstant teacher-relative variation must pass through the latent pathway. This does not automatically certify that every property of the reconstruction \(\hat{x}\) is semantically faithful, because a powerful decoder may reconstruct through paths that bypass the raw witness.

A natural extension is therefore to add a reconstruction-side proxy \(T(\hat{x})\). A future triadic protocol could monitor consistency among \(T(x)\), \(S(z)\), and \(T(\hat{x})\), for example through a loss of the form
\[
L_{\mathrm{triad}}
=
\mathrm{KL}\!\left(T(x)\,\|\,S(z)\right)
+
\mathrm{KL}\!\left(S(z)\,\|\,T(\hat{x})\right)
+
\mathrm{KL}\!\left(T(x)\,\|\,T(\hat{x})\right).
\]
Such a protocol would explicitly test whether reconstruction-side semantics remain consistent with both the input-side teacher and the latent-side witness. It could expose decoder paths that reconstruct well while bypassing the raw z-only witness.

The same proxy-layer view also explains why trajectory-level control is plausible. In free logit or free energy space, if
\[
L(u)=\sum_i \mathrm{KL}\!\left(T_i\,\|\,\mathrm{softmax}(u_i)\right)
\]
and the logits follow continuous-time negative gradient flow \(\dot u=-\nabla_u L\), then
\[
\frac{dL}{dt}
=
\langle \nabla_u L,\dot u\rangle
=
-\|\nabla_u L\|^2
\le 0.
\]
Thus \(L^{raw}_{\mathrm{align}}\) is analytically monotone in the unconstrained proxy space. In the actual neural parameterization, this monotonicity is no longer automatic: the encoder \(z_\phi(x)\), decoder, minibatch noise, and shared parameters can all perturb the free-logit descent geometry. This motivates tentative updates, adaptive learning rates, and rollback-style guards that monitor \(G_\tau\), \(L_{\mathrm{align}}\), student MI, linear probes, PSNR, and alignment trajectories after a run enters the safe region \(G_\tau>0\).

In short, \(T\), \(S\), and eventually \(T(\hat{x})\) provide a visible, controllable, finite-dimensional layer on top of a high-dimensional generative process. The present paper establishes the first certified edge of this layer; extending it to reconstruction-side consistency and trajectory-level guarantees is an open direction.

\clearpage
\IfFileExists{./checklist.tex}{\input{./checklist.tex}}{}


\clearpage
{\small
\begin{thebibliography}{99}
\bibitem[Alemi et~al.(2018)]{alemi2018}
Alemi, A., Poole, B., Fischer, I., Dillon, J., Saurous, R., and Murphy, K. Fixing a broken ELBO. \emph{ICML}, 2018.

\bibitem[Bowman et~al.(2016)]{bowman2016}
Bowman, S., Vilnis, L., Vinyals, O., Dai, A., Jozefowicz, R., and Bengio, S. Generating sentences from a continuous space. \emph{CoNLL}, 2016.

\bibitem[Chen et~al.(2017)]{chen2017}
Chen, X., Kingma, D., Salimans, T., Duan, Y., Dhariwal, P., Schulman, J., Sutskever, I., and Abbeel, P. Variational lossy autoencoder. \emph{ICLR}, 2017.

\bibitem[Dilokthanakul et~al.(2016)]{dilokthanakul2016}
Dilokthanakul, N., Mediano, P., Garnelo, M., Lee, M., Salimbeni, H., Arulkumaran, K., and Shanahan, M. Deep unsupervised clustering with Gaussian mixture variational autoencoders. arXiv:1611.02648, 2016.

\bibitem[He et~al.(2019)]{he2019lagging}
He, J., Spokoyny, D., Neubig, G., and Berg-Kirkpatrick, T. Lagging inference networks and posterior collapse in variational autoencoders. \emph{ICLR}, 2019.

\bibitem[Hinton et~al.(2015)]{hinton2015}
Hinton, G., Vinyals, O., and Dean, J. Distilling the knowledge in a neural network. arXiv:1503.02531, 2015.

\bibitem[Kingma and Welling(2014)]{kingma2014}
Kingma, D. and Welling, M. Auto-encoding variational Bayes. \emph{ICLR}, 2014.

\bibitem[Krizhevsky(2009)]{krizhevsky2009cifar}
Krizhevsky, A. Learning multiple layers of features from tiny images. Technical report, University of Toronto, 2009.

\bibitem[Le and Yang(2015)]{le2015tiny}
Le, Y. and Yang, X. Tiny ImageNet visual recognition challenge. CS 231N, 2015.

\bibitem[van den Oord et~al.(2017)]{oord2017}
van den Oord, A., Vinyals, O., and Kavukcuoglu, K. Neural discrete representation learning. \emph{NeurIPS}, 2017.

\bibitem[Zhao et~al.(2019)]{zhao2019}
Zhao, S., Song, J., and Ermon, S. InfoVAE: information maximizing variational autoencoders. \emph{AAAI}, 2019.
\end{thebibliography}}
\end{document}